%% file: main.tex
\documentclass{article} 
\usepackage{geometry}
\usepackage[authoryear]{natbib}

\usepackage{hyperref}
\usepackage{url}
\input{macros}
\usepackage{adjustbox} 

\usepackage{listings}
\usepackage{xcolor}
\definecolor{darkblue}{rgb}{0.0,0.0,0.65}
\definecolor{darkred}{rgb}{0.68,0.05,0.0}
\definecolor{darkgreen}{rgb}{0.0,0.29,0.29}
\definecolor{darkpurple}{rgb}{0.47,0.09,0.29}
\hypersetup{
   colorlinks = true,
   citecolor  = darkblue,
   linkcolor  = darkred,
   filecolor  = darkblue,
   urlcolor   = darkblue,
 }

\begin{document}


\title{Near-Optimal algorithms for group distributionally robust optimization and beyond}
\begin{center}
\begin{LARGE}
  \textbf{Near-Optimal algorithms for group distributionally robust optimization and beyond}
\end{LARGE}
\vskip 1cm
\begin{minipage}[b]{6.5cm}
  \centering
  \textbf{Tasuku Soma }\\
  Institute of Statistical Mathematics\\
  \texttt{soma@ism.ac.jp}
\end{minipage}
%
\hskip .75cm
\begin{minipage}[b]{6.5cm}
  \centering
  \textbf{Khashayar Gatmiry}\\
  MIT CSAIL\\
  \texttt{gatmiry@mit.edu}
\end{minipage}
%
\hskip .75cm
\vskip 1cm
\begin{minipage}[b]{6.5cm}
  \centering
  \textbf{Sharut Gupta}\\
  MIT CSAIL\\
  \texttt{sharut@mit.edu}
\end{minipage}
%
\hskip .75cm
\begin{minipage}[b]{6.5cm}
  \centering
  \textbf{Stefanie Jegelka}\\
  MIT CSAIL\\
  \texttt{stefje@mit.edu}
\end{minipage}
\vskip 1cm
\end{center}


\begin{abstract}
Distributionally robust optimization (DRO) can improve the robustness and fairness of learning methods. In this paper, we devise stochastic algorithms for a class of DRO problems including group DRO, subpopulation fairness, and empirical conditional value at risk (CVaR) optimization. Our new algorithms achieve faster convergence rates than existing algorithms for multiple DRO settings. We also provide a new information-theoretic lower bound that implies our bounds are tight up to a log factor for group DRO. Empirically, too, our algorithms outperform known methods.
\end{abstract}

\input{intro}
\input{preliminaries}
\input{examples}
\input{algorithms}
\input{lower}
\input{experiments}

\clearpage
\bibliographystyle{plainnat}
\bibliography{main}
\clearpage
\appendix
\input{appendix}

\end{document}

%% file: macros.tex
\usepackage{amsmath,amssymb,amsthm}

\newtheorem{theorem}{Theorem}
\newtheorem{lemma}{Lemma}
\newtheorem{assumption}{Assumption}

\usepackage{multirow}
\usepackage{algorithm}
\usepackage[noend]{algorithmic}
\usepackage{pgfplots}
\pgfplotsset{compat=1.17}
\usepgfplotslibrary{colorbrewer,groupplots}
\usepackage{subcaption}
\DeclareMathOperator*{\E}{\mathbf{E}}
\DeclareMathOperator{\proj}{proj}
\DeclareMathOperator*{\argmin}{argmin}
\DeclareMathOperator*{\argmax}{argmax}
\DeclareMathOperator*{\sign}{sign}
\DeclareMathOperator*{\diag}{diag}

\newcommand{\eps}{\varepsilon}
\newcommand{\R}{\mathbb{R}}
\newcommand{\bfe}{\mathbf{e}}
\newcommand{\bfone}{\mathbf{1}}
\newcommand{\caA}{\mathcal{A}}
\newcommand{\caL}{\mathcal{L}}
\newcommand{\caP}{\mathcal{P}}
\newcommand{\caO}{\mathcal{O}}
\newcommand{\Ber}{\mathrm{Ber}}
\newcommand{\DKL}{D_\mathrm{KL}}
\newcommand{\dTV}{d_\mathrm{TV}}
\usepackage{mathtools}
\DeclarePairedDelimiter{\norm}{\lVert}{\rVert}
\DeclarePairedDelimiter{\abs}{\lvert}{\rvert}

\usepackage{nicematrix}
\usepackage{booktabs}
\newcommand{\specialcell}[2][c]{%
  \begin{tabular}[#1]{@{}c@{}}#2\end{tabular}}

%% file: intro.tex
\section{Introduction}
Commonly, machine learning models are trained to optimize the average performance.
However, such models may not perform equally well among all demographic subgroups due to a hidden bias in the training set or distribution shift in training and test phases~\citep{Hovy2015, Hashimoto18a, Martinez2021, Duchi2021}.
Biases in datasets are also directly related to fairness concerns in machine learning~\citep{Buolamwini18a, Jurgens2017}.

Recently, various algorithms based on distributionally robust optimization (DRO) have been proposed to address these problems~\citep{Hovy2015,Hashimoto18a,Hu2018,Oren2019,Williamson2019a,Sagawa2020,Curi2020,Zhang2021,Martinez2021,Duchi2021}.
However, these algorithms are often highly tailored to each specific DRO formulation. Furthermore, it is often unclear whether these proposed algorithms are optimal in terms of the convergence rate.
Are there a unified algorithmic methodology and a lower bound for these problems?

\paragraph{Contributions.}
In this paper, we study a general class of DRO problems, which includes group DRO~\citep{Hu2018,Oren2019,Sagawa2020}, subpopulation fairness~\citep{Martinez2021}, conditional value at risk (CVaR) optimization~\citep{Curi2020}, and many others.
Let $\Theta \subseteq \R^n$ be a convex set of model parameters and $\ell(\theta; z) : \Theta \to \R_+$ be a convex loss of the model with parameter $\theta$ with respect to data point $z$.
The data point $z$ may be drawn from one out of $m$ distributions $P_1, \dots, P_m$ which are accessible via a stochastic oracle that returns an i.i.d.~sample $z \sim P_i$.
Let $Q$ be a convex subset of the probability simplex in $\R^m$ that contains the uniform vector, i.e., $(1/m, \dots, 1/m) \in Q$.
In this paper, we conside the following DRO
\begin{align}\label{eq:general-DRO}
    \min_{\theta \in \Theta}\max_{q \in Q} \sum_{i=1}^m q_i\E_{z \sim P_i}[\ell(\theta; z)],
\end{align}
which we call \emph{generalized group DRO}.
If $Q$ are the probability simplex and scaled $k$-set polytope, we can recover group DRO~\citep{Sagawa2020} and subpopulation fairness~\citep{Martinez2021}, respectively.
Moreover, we formulate a new, more general fairness concept based on weighted rankings with $Q$ being a permutahedron, which includes these special cases; see Section~\ref{sec:examples} for details.

For generalized group DRO, we devise an efficient stochastic gradient algorithm.
Furthermore, we show that it almost achieves the information-theoretic optimal convergence rate for group DRO up to a log factor.
Our main technical contributions are as follows;

\begin{table*}[t]
    \centering
    \caption{Summary of convergence results for group DRO. Here, $m$ denotes the number of groups, $n$ the dimension of $\theta$, $G$ the Lipschitz constant of loss function $\ell$, $D$ the diameter of feasible set $\Theta$, $M$ the range of loss function $\ell$, and $T$ the number of calls to stochastic oracle. The convergence of \citep{Sagawa2020} and Theorem~\ref{thm:GDRO-EXP3} are with respect to $\E[\eps_T]$ while the convergence of Theorem~\ref{thm:GDRO-TINF} is a weaker bound with respect to $\E[\eps_T(q^*)]$ for a saddle point $(\theta^*, q^*)$.} \label{tab:main}
    \begin{tabular}{c|cc|c}
        reference & \specialcell{convergence rate \\ $\E[\eps_T]^\dagger$ or $\E[\eps_T(q^*)]^{\ddagger}$} & iteration complexity & lower bound  \\\hline
        \citep{Sagawa2020} & $O\Big(m\sqrt{\frac{G^2D^2 + M^2\log m}{T}}\Big)^{\dagger}$ & $O(m+n)$ + proj.~onto $\Theta$ & \multirow{3}{*}[-2em]{\parbox{6em}{$\Omega\Big(\sqrt{\frac{G^2D^2 + M^2m}{T}}\Big)$\\ \bf (Theorem~\ref{thm:lb})}}\\[1em]
        \bf Ours (Theorem~\ref{thm:GDRO-EXP3}) & $O\Big(\sqrt{\frac{G^2D^2 + M^2m \log m}{T}}\Big)^{\dagger}$  & $O(m+n)$ + proj.~onto $\Theta$ \\[1em]
        \bf Ours (Theorem~\ref{thm:GDRO-TINF}) & $O\Big(\sqrt{\frac{G^2D^2 + M^2m}{T}}\Big)^{\ddagger}$ & \parbox{11em}{$O(m+n)$ + proj.~onto $\Theta$ \\ + solving scalar equation} \\
    \end{tabular}
\end{table*}

\begin{itemize}
    \item We provide a generic stochastic gradient algorithm for generalized group DRO.
    By specializing it in the group DRO setting, we provide two algorithms (\textsc{GDRO-EXP3} and \textsc{GDRO-TINF}) that improve the rate of \citet{Sagawa2020} by a factor of $\Omega(\sqrt{m})$ with the almost same complexity per iteration; see Table~\ref{tab:main}.
    Furthermore, our generic algorithm can be specialized to improve the convergence rate of \cite{Curi2020} for subpopulation fairness (a.k.a. empirical CVaR optimization).
    Finally, we show that our algorithm runs efficiently if $Q$ is a permutahedron, which includes all aforementioned subclasses. 

    \item We prove an almost matching information-theoretic lower bound for the convergence rate of group DRO.
    This implies that no algorithm can improve the convergence rate of \textsc{GDRO-EXP3} (up to a constant factor).
    To the best of our knowledge, this is the first information-theoretic lower bound for group DRO.

    \item Our experiments on real-world and synthetic datasets show that our algorithms also empirically outperform the known algorithm, supporting our theoretical analysis. Although our convergence analysis only holds for the convex regime, our proposed algorithms outperform even in the deep learning regime.
\end{itemize}

\subsection{Our techniques}
\textbf{Algorithms.}
The core idea of our algorithms is \emph{stochastic no-regret dynamics}~\citep{Hazan2016book}.
We regard DRO~\eqref{eq:general-DRO} as a two-player zero-sum game between a player who picks $\theta \in \Theta$ and another player who picks $q \in Q$.
The two players iteratively update their solution using online learning algorithms;
in particular, we will use online gradient descent (OGD)~\citep{Zinkevich2003} and online mirror descent (OMD)~\citep{Cesa-Bianchi2006} for the $\theta$-player and $q$-player, respectively.
In addition, we need to estimate gradients for both players, since the objective function of generalized group DRO is stochastic and we cannot obtain exact gradients.

The convergence rate of stochastic no-regret dynamics depends on the expected regret of OGD and OMD.
To obtain a near-optimal convergence rate, we must carefully choose the regularizer in OMD as well as gradient estimators, exploiting the structure of generalized group DRO.
In particular, we need to balance the variance of gradient estimators and the diameter terms in \emph{both} OGD and OMD.
This is the most challenging part of the algorithm design.
Inspired by adversarial multi-armed bandit algorithms, we design gradient estimators for no-regret dynamics of OGD and OMD in generalized group DRO.
Indeed, our algorithms for group DRO (\textsc{GDRO-EXP3} and \textsc{GDRO-TINF}) are based on adversarial multi-armed bandit algorithms, EXP3~\citep{Auer2003} and Tsallis-INF~\citep{Zimmert2021}, respectively, hence the name.
Although each building block (OGD, OMD, and gradient estimators) is fairly known in the literature, we need to put them together in the right combination to obtain the correct rate.

\textbf{Lower bound.}
%
For the lower bound, we carefully design a family of group DRO instances for which any algorithm requires a certain number of queries to achieve a good objective value.
To bound the number of queries, we use information-theoretic tools such as Le Cam's lemma and bound the Kullback-Leibler divergence between Bernoulli distributions. Such tools are also used at the heart of lower bounds for stochastic convex optimization~\citep{Agarwal2012} and adversarial multi-armed bandits~\citep{Auer2003}, but the connection to those settings is much more subtle here, and our construction is specifically designed for group DRO-type problems.

\subsection{Related work}
DRO is a wide field ranging from robust optimization to machine learning and statistics~\citep{Goh2010,Bertsimas2018}, whose original idea dates back to \citet{Scarf1958}.
Popular choices of the uncertainty set in DRO include balls around an empirical distribution in Wasserstein distance~\citep{Esfahani2018,Blanchet2019}, $f$-divergence~\citep{Namkoong2016,Duchi2021}, $\chi^2$-divergence~\citep{Staib19a}, and maximum mean discrepancy~\citep{Staib19b,Kirschner20a}.

DRO algorithms have been mainly studied for the offline setting, i.e., algorithms can access all data points of the empirical distribution.
Note that generalized group DRO is not offline because the group distributions are given by the stochastic oracles.
\citet{Namkoong2016} proposed stochastic gradient algorithms for offline DRO with $f$-divergence uncertainty sets.
\citet{Curi2020} used no-regret dynamics for empirical CVaR minimization.
Their algorithm invokes sampling from $k$-DPP in each iteration, which is more computationally demanding than our algorithm.
Furthermore, our algorithm gets rid of an $O(\log m)$ factor in the convergence rate using the Tsallis entropy regularizer; see Theorem~\ref{thm:permuta}.
\citet{Qi2021,Jikai2021} devised stochastic gradient algorithms for several DRO with non-convex losses.

\citet{Agarwal2012} gave a lower bound for stochastic convex optimization, which is a special case of generalized group DRO with only one distribution.
Recently, \citet{Carmon21a} showed a lower bound for minimax problem $\min_{x} \max_{i = 1}^m f_i(x)$ for non-stochastic Lipschitz convex $f_i$.
Our lower bound deals with the stochastic functions, so this result does not apply.

In this paper, we assume that the group information is given in advance.
However, the group information might not be easy to define in practice.
\citet{Bao2021} propose a simple method to define groups for classification problems based on mistakes of models in the training phase.
Their method often generates group DRO instances with large $m$.
Our algorithms are more efficient for such group DRO thanks to the better dependence on $m$ in the convergence rate.

No-regret dynamics is a well-studied method for solving two-player zero-sum games~\citep{Cesa-Bianchi2006}.
For non-stochastic convex-concave games, one can achieve $O(1/T)$ convergence via predictable sequences~\citep{Rakhlin2013a}.
This result does not apply to our setting because our DRO is a stochastic game.

After we submitted the first version of the present paper, there appeared an independent work~\citep{Haghtalab2022} which also studies group DRO and related problems. They obtained the same convergence rate for these problems using a similar approach based on stochastic no-regret dynamics. On the other hand, the present paper studies a more general class of DRO problems and draws a connection to various fairness concepts.

%% file: preliminaries.tex
\paragraph{Notations.}
Throughout the paper, $m$ denotes the number of distributions (groups) and $n$ denotes the dimension of a variable $\theta$.
For a positive integer $m$, we write $[m] := \{1, \dots, m\}$.
The orthogonal projection onto set $\Theta$ is denoted by $\proj_\Theta$.
The $i$th standard unit vector is denoted by $\bfe_i$ and the all-one vector is denoted by $\bfone$.
The probability simplex in $\R^m$ is denoted by $\Delta_m$.


%% file: examples.tex
\section{Examples contained in generalized group DRO}\label{sec:examples}
In this section, we show how several DRO formulations in the literature can be phrased in generalized group DRO~\eqref{eq:general-DRO}.
In addition, we propose a novel fairness constraint based on weighted rankings using generalized group DRO.

\paragraph{Group DRO.}
When $Q$ equals the probablility simplex, we obtain original group DRO~\citep{Hu2018,Oren2019,Sagawa2020}:
\begin{align}\label{eq:group-DRO}
    \min_{\theta \in \Theta} \max_{i=1}^m\; \E_{z \sim P_i}[\ell(\theta ; z)].
\end{align}
That is, group DRO aims to minimize the expected loss in the worst group, thereby ensuring better performance across all groups.

\paragraph{Empirical CVaR, Subpopulation fairness, Average top-$k$ worst group loss.}
Group DRO may yield overly pessimistic solutions.
For instance, the groups might be automatically generated by other algorithms (such as one in \cite{Bao2021}) and there might exist a few ``outlier'' groups that make the group DRO objective trivial.

For such a case, we can restrict $Q$ to a small subset of the probability simplex so that the solution cannot put large weights on a few outlier groups.
Especially, let
\[
    Q = \left\{q \in \Delta_m : 0 \leq q_i \leq \frac{1}{pm} \right\}
\]
for some parameter $p \in (0,1)$, i.e., $Q$ is a scaled $k$-set polytope.
The intuition behind the choice of $Q$ is that, by limiting the largest entry of $q$ to $1/p m$, DRO would optimize the expected loss over the worst $p$-fraction subgroups of $m$ groups.
Therefore, if the fraction of outlier groups is sufficiently small compared to $p$, then $p$-fraction subgroups must contain ``inlier'' groups as well.
Therefore, it is likely that DRO with $Q$ finds solutions more robust than group DRO.

When $P_i$ is the Dirac measure of data $z_i$, then the resulting DRO is empirical CVaR optimization~\citep{Curi2020}.
In the fairness context, the same problem is called subpopulation fairness~\citep{Williamson2019a,Martinez2021,Duchi2021}.

If $p = k/m$ for some positive integer $k$, the resulting DRO is the average top-$k$ worst group loss~\citep{Zhang2021}:
\[
    \min_{\theta \in \Theta} \frac{1}{k} \sum_{i=1}^k L_i^{\downarrow}(\theta),
\]
where $L_i^{\downarrow}(\theta)$ denotes the the $i$th largest population group loss of $\theta$.
More precisely, let $L_i(\theta) = \E_{z \sim P_i}[\ell(\theta; z)]$ for $i \in [m]$ and sort them in the non-increasing order: $L_1^\downarrow(\theta) \geq \dots \geq L_m^\downarrow(\theta)$.

\paragraph{Weighted ranking of group losses.}
The aforementioned DRO formulations are special cases of the following DRO, which we call the \emph{weighted ranking of group losses.}
Let $\alpha \in \Delta^m$ be a fixed vector with non-increasing entries.
Let $Q$ be the permutahedron of $\alpha$, the convex hull of $(\alpha_{\sigma(1)}, \dots, \alpha_{\sigma(m)})$ for all permutations $\sigma$ of $[m]$.
Then, the resulting DRO is
\[
    \min_{\theta \in \Theta}\sum_{i=1}^m \alpha_i L_i^{\downarrow}(\theta).
\]
Group DRO corresponds to $\alpha = (1, 0, \dots, 0)$ and the average top-$k$ worst group losses corresponds to $\alpha = (\underbrace{1/k, \dots, 1/k}_{\text{$k$ times}}, 0, \dots, 0)$.
Another example that is contained in none of the above examples is \emph{lexicographic minimax fairness}~\citep{Diana2021}.
The goal of lexicographical minimax fairness is to find $\theta \in \Theta$ such that the sequence $(L^\downarrow_1(\theta), \dots, L^\downarrow_m(\theta))$ is lexicographically minimum.
This corresponds to $\alpha$ with sufficiently varied entries, i.e., $\alpha_1 \gg \alpha_2 \gg \dots \gg\alpha_m$.

%% file: algorithms.tex
\section{Algorithms}
In this section, we describe our algorithms.
First, we present a generic algorithm for generalized group DRO~\eqref{eq:general-DRO}  and provide a unified convergence analysis in Section~\ref{subsec:alg-general}.
Then, we specialize it into two concrete algorithms for group DRO~\eqref{eq:group-DRO} in Section~\ref{subsec:alg-GDRO}.
We sketch algorithms for the average of top-$k$ group losses and weighted ranking of group loss in Section~\ref{subsec:alg-top-k}.

\subsection{Algorithm for the general case}\label{subsec:alg-general}
We present our algorithm for generalized group DRO~\eqref{eq:general-DRO}.
At a high level, our algorithm can be regarded as stochastic no-regret dynamics.
Let us denote
$
    L(\theta, q) := \sum_{i=1}^m q_i \E_{z \sim P_i}[\ell(\theta; z)].
$
Imagine that the $\theta$-player and $q$-player run online algorithms $\caA_\theta$ and $\caA_q$, respectively, to solve the minimax problem $\min_{\theta \in \Theta}\max_{q\in Q} L(\theta, q)$.
That is, for $t = 1, \dots, T$,
\begin{itemize}
    \item $\theta_{t} \in \Theta$ and $q_t \in Q$ are determined by $\caA_\theta$ and $\caA_q$, respectively.
    \item Both players feed gradient estimators $\hat\nabla_{\theta,t}$ and $\hat\nabla_{q,t}$ to $\caA_\theta$ and $\caA_q$, respectively. Here, $\E[\hat\nabla_{\theta,t}] = \nabla_\theta L(\theta_t, q_t)$ and $\E[\hat\nabla_{q,t}] = \nabla_q L(\theta_t, q_t)$.
\end{itemize}

Let $\theta^*$ be an optimal solution.
Let
\[
    \eps_T := \max_{q \in Q} L(\bar\theta_{1:T}, q) - \max_{q \in Q} L(\theta^*, q).
\]
be the optimality gap of the averaged iterate $\bar\theta_{1:T} = \frac{1}{T} \sum_{t=1}^T \theta_t$.

We can bound the expected convergence rate $\E[\eps_T]$ via regrets $R_\theta$ and $R_q$ of these online algorithms (see Appendix~\ref{app:preliminaries} for a formal definition), i.e.,
\begin{align}\label{eq:converge-regret}
    \E[\eps_T] \leq \frac{\E[R_\theta(T; \theta^*)] + \E[R_q(T)]}{T}.
\end{align}
We can obtain hence the convergence rate of the above algorithms by investigating the expected regret bounds of these online algorithms.

We also use the following weaker notion of convergence.
For any \emph{fixed} sabble point $(\theta^*, q^*) \in \Theta \times Q$ of the problem~\eqref{eq:general-DRO}, let
\[
    \eps_T(q^*) := L(\bar\theta_{1:T}, q^*) - L(\theta^*, q^*),
\]
be the gap of $\bar\theta_{1:T}$ with respect to $(\theta^*, q^*)$.
Similar to \eqref{eq:converge-regret}, we can bound 
\begin{align}\label{eq:regret-to-saddle}
    \E[\eps_T(q^*)] \leq \frac{\E[R_\theta(T: \theta^*)] + \E[R_q(T; q^*)]}{T}.
\end{align}

To get a concrete algorithm, we must specify the online algorithms $\caA_\theta, \caA_q$ as well as the gradient estimators $\hat\nabla_{\theta, t}, \hat\nabla_{q,t}$.
We use OGD and OMD as $\caA_\theta$ and $\caA_q$, respectively.
We construct the gradient estimators by sampling $i_t \sim q_t$ and $z \sim P_{i_t}$ and setting $\hat\nabla_{\theta, t} = \nabla_\theta\ell(\theta_t; z)$ and $\hat\nabla_{q,t} = \frac{\ell(\theta_t; z)}{q_{t,i_t}}\bfe_{i_t}$.
This leads to Algorithm~\ref{alg:general}.
There, $\Psi: Q \to \R$ denotes the regularizer of OMD and $\eta_{\theta, t}$ and $\eta_q$ denote the step sizes of OGD and OMD, respectively.\footnote{We make a standard assumption that the regularizer $\Psi$ is differentiable and strictly convex, and satisfies $\norm{\nabla \Psi(x)} \to +\infty$ as $x$ tends to the boundary of $Q$.}
It turns out that this combination of online algorithms and gradient estimators yields the best convergence rate (for group DRO) because the expected regrets of both players are optimal.

\begin{algorithm}[h]
\caption{Algorithm for generalized group DRO~\eqref{eq:general-DRO}}\label{alg:general}
    \begin{algorithmic}[1]
        \REQUIRE initial solution $\theta_1 \in \Theta$, number of iterations $T$, step sizes $\eta_{\theta, t} > 0$ ($t \in [T]$), $\eta_q > 0$, and a strictly convex function $\Psi : Q \to \R$.
        \STATE Let $q_1 = (1/m, \dots, 1/m)$.
        \FOR{$t = 1, \dots, T$}
            \STATE Sample $i_t \sim q_t$.
            \STATE Call the stochastic oracle to obtain $z \sim P_{i_t}$.
            \STATE $\theta_{t+1} \gets \proj_\Theta(\theta_t - \eta_{\theta, t}\nabla_\theta\ell(\theta_t; z))$
            \STATE $\nabla\Psi(\tilde q_{t+1}) \gets \nabla\Psi(q_t) - \frac{\eta_{q}}{q_{t,i_t}}\ell(\theta_t; z)\bfe_{i_t}$; $q_{t+1} \gets \argmin_{q\in Q} D_\Psi(q, \tilde q_{t+1})$, where $D_\Psi(x, y) = \Psi(x) - \Psi(y) - \nabla\Psi(x)^\top(y - x)$ is the Bregman divergence with respect to~$\Psi$.
        \ENDFOR
        \RETURN $\frac{1}{T}\sum_{t=1}^T \theta_t$.
    \end{algorithmic}
\end{algorithm}

We now analyze the convergence rate of Algorithm~\ref{alg:general}.
We make the following standard assumptions.
\begin{assumption}\label{assump}
    The loss function $\ell(\theta; z)$ is continuously differentiable and $G$-Lipchitz in $\theta$, and has range $[0,M]$ for all $z$.
    The Euclidean diameter of the feasible region $\Theta$ is at most $D$.
\end{assumption}
The following theorem follows from plugging regret bounds of OGD and OGD, and the construction of the gradient estimators into \eqref{eq:converge-regret}.
\begin{theorem}\label{thm:general}
    If $\eta_{\theta, t}$ is nonincreasing,
    Algorithm~\ref{alg:general} achieves the expected convergence rate
    \[
        \E[\eps_T(q^*)]
        \leq \frac{1}{T} \left(\frac{G^2}{2}\sum_{t=1}^T \eta_{\theta,t} + \frac{D^2}{2\eta_{\theta, T}} 
        + \frac{M^2}{2}\eta_{q} \sum_{t=1}^T \E_{i_t} \left[ \frac{(\nabla^2\Psi(q_t))^{-1}_{i_t,i_t}}{q_{t,i_t}^2} \right] + \frac{D_\Psi(q^*, \bfone / m)}{\eta_{q}} \right).
        \]
    for any fixed saddle point $(\theta^*, q^*)$.
\end{theorem}
A formal proof can be found in Appendix~\ref{app:proofs}.
We will see how specific choices of the regularizer $\Psi$ yield various algorithms and convergence rates for group DRO and others in the next subsections.
A few remarks on the regularizers, step sizes, and projection step are in order.

\paragraph{Regularizer.}
Although Algorithm~\ref{alg:general} works with general $\Psi$, we can choose a specific regularizer for $Q$ appearing in applications, e.g, the probability simplex, scaled $k$-set polytope, or a permutahedron. In the next subsections, we show that the entropy regularizer $\Psi(x) = \sum_i (x_i \log x_i - x_i)$ and Tsallis entropy regularizer $\Psi(x) = 2(1 - \sum_i \sqrt{x_i})$ yield efficient algorithms with improved convergence rates for these cases.

\paragraph{Step sizes.} The theorem includes decreasing step sizes such as $\eta_{\theta,t} = \frac{D}{mG\sqrt{t}}$ in addition to fixed step sizes.
Decreasing step sizes have the advantage that we do not require the knowledge of $T$ at the beginning of the algorithm but come at the cost of an extra constant factor in the expected convergence rate.
Since both step size policies give the asymptotically same convergence rate, we describe only fixed step sizes in the theorems in the next subsections.
In practice, decreasing step sizes stabilize the algorithm and often outperform fixed step sizes.

\paragraph{Projection step.}
In general, the Bregman projection $\argmin_{q\in Q} D_\Psi(q, \tilde q_{t+1})$ is convex, but may be costly to compute.
For the applications described in Section~\ref{sec:examples}, $Q$ is a permutahedron. In this case, it is known that the Bregman projection with respect to the entropy and Tsallis entropy regularizers can be done in $O(m \log m)$ time~\citep{Lim2016}.
If $Q$ is the probability simplex, we even have a closed form for the Bregman projection.

\subsection{Algorithms for Group DRO}\label{subsec:alg-GDRO}
We now describe two concrete algorithms for group DRO~\eqref{eq:group-DRO}.

\paragraph{\textsc{GDRO-EXP3P}.} 
The first algorithm is obtained by using the EXP3P algorithm~\cite{} for the $q$-player algorithm. The resulting algorithm, \textsc{GDRO-EXP3P}, is shown in Algorithm~\ref{alg:GDRO-EXP3}.
The update is in a closed formula and its complexity is $O(m+n)$ time.
The convergence rate follows from Theorem~\ref{thm:general}. 

\begin{algorithm}
\caption{\textsc{GDRO-EXP3P}}\label{alg:GDRO-EXP3}
    \begin{algorithmic}[1]
        \REQUIRE initial solution $\theta_1 \in \Theta$, number of iterations $T$, and step sizes $\eta_{\theta, t} > 0$ ($t \in [T]$), $\eta_q > 0$, $\beta, \gamma > 0$.
        \STATE Let $q_1 = (1/m, \dots, 1/m)$.
        \FOR{$t = 1, \dots, T$}
            \STATE Sample $i_t \sim q_t$.
            \STATE Call the stochastic oracle to obtain $z \sim P_{i_t}$.
            \STATE $\theta_{t+1} \gets \proj_\Theta(\theta_t - \eta_{\theta, t}\nabla_\theta\ell(\theta_t; z))$
            \STATE Let $\tilde g_t := \frac{-\ell(\theta_t; z)\bfe_{i_t} + \beta\bfone}{q_{t,i_t}}$, $G_{t} := \sum_{\tau=1}^t \tilde g_t$, and $Z := \sum_{i \in [m]} \exp(\eta G_{t,i})$.
            \STATE $q_{t+1} \gets (1-\gamma)\frac{\exp(\eta G_t)}{Z} + \frac{\gamma\bfone}{m}$.
        \ENDFOR
        \RETURN $\frac{1}{T}\sum_{t=1}^T \theta_t$.
    \end{algorithmic}
\end{algorithm}

\begin{theorem}\label{thm:GDRO-EXP3}
    If $\eta_{\theta, t}$ is nonincreasing,
    \textsc{GDRO-EXP3P} (Algorithm~\ref{alg:GDRO-EXP3}) achieves 
    \begin{align}\label{eq:GDRO-EXP3}
        \E[\eps_T]
        \leq \frac{1}{T} \left(\frac{G^2}{2}\sum_{t=1}^T \eta_{\theta,t} + \frac{D^2}{2\eta_{\theta, T}} + \frac{mM^2}{2}\eta_{q}T + \frac{\log m}{\eta_{q}} \right).
    \end{align}
    For $\eta_{\theta, t} = \frac{D}{G\sqrt{T}}$ and $\eta_{q} = \sqrt{\frac{2\log m}{mM^2 T}}$, we obtain
    \[
        \E[\eps_T] \leq \sqrt{2}\frac{\sqrt{G^2D^2 + 2M^2 m \log m}}{\sqrt T}.
    \]
\end{theorem}

\paragraph{Comparison to \citet{Sagawa2020}.}
Our algorithm improve the convergence rate of \citet{Sagawa2020} by a factor of $O(\sqrt{m})$; see Table~\ref{tab:main}.
The reason lies in the choice of gradient estimator. All algorithms are stochastic no-regret dynamics. As outlined above, their convergence hence can be bounded by the regrets of the players, which depend on the variance of the local norm of the gradient estimators. Their strategy is based on uniform sampling that yields a variance of $O(m)$ for both players, whereas our bound is $O(\sqrt{m})$ thanks to the gradient estimators tailored to the regularizer of OMD.
%
%
More details may be found in Appendix~\ref{app:sagawa}.

\paragraph{\textsc{GRDO-TINF}.}
The second algorithm is given by using the Tsallis entropy regularizer for the $q$-player algorithm. 
The update of $q_t$ is now
\begin{alignat*}{3}
    \tilde q_{t+1} &= q_t \left(\bfone - \frac{\eta_{q}\sqrt{q_t}}{q_{t,i_t}}\ell(\theta_t; z)\bfe_{i_t} \right)^{-2}, \\
    q_{t+1} &:= \left( \frac{1}{\sqrt{\tilde q_{t+1}}}  - \alpha\bfone \right)^{-2},
\end{alignat*}
where the multiplication, square-root, and power operations are entry-wise and $\alpha \in \R$ is the unique solution of equation
$
    \sum_{i=1}^m \left(1/\sqrt{\tilde q_{t+1, i}} - \alpha \right)^{-2} = 1.
    $
The solution $\alpha$ can be computed via the Newton method.
Practically, one can use $\alpha$ in the previous iteration to warm start the Newton method.
In each iteration, the algorithm performs a single orthogonal projection onto $\Theta$, the Newton method for finding $\alpha$, and $O(m+n)$ operations to update $\theta_t, q_t$.
The pseudocode is given in Algorithm~\ref{alg:GDRO-TINF}.

\begin{algorithm}
\caption{\textsc{GDRO-TINF}}\label{alg:GDRO-TINF}
    \begin{algorithmic}[1]
        \REQUIRE initial solution $\theta_1 \in \Theta$, number of iterations $T$, and step sizes $\eta_{\theta, t} > 0$ ($t \in [T]$), $\eta_q > 0$.
        \STATE Let $q_t = (1/m, \dots, 1/m)$.
        \FOR{$t = 1, \dots, T$}
            \STATE Sample $i_t \sim q_t$.
            \STATE Call the stochastic oracle to obtain $z \sim P_{i_t}$.
            \STATE $\theta_{t+1} \gets \proj_\Theta(\theta_t - \eta_{\theta, t}\nabla_\theta\ell(\theta_t; z))$
            \STATE $\tilde q_{t+1} \gets q_t \left(\bfone - \frac{\eta_{q}\sqrt{q_t}}{q_{t,i_t}}\ell(\theta_t; z)\bfe_{i_t} \right)^{-2}$
            \STATE Compute $\alpha \in \R$ such that $\sum_{i=1}^m \left(1 / \sqrt{\tilde q_{t+1, i}} - \alpha \right)^{-2} = 1$.
            \STATE $q_{t+1} \gets \left(\tilde q_{t+1}^{-1/2} - \alpha \bfone \right)^{-2}$
        \ENDFOR
        \RETURN $\frac{1}{T}\sum_{t=1}^T \theta_t$.
    \end{algorithmic}
\end{algorithm}

From Theorem~\ref{thm:general}, we obtain the following convergence rate.

\begin{theorem}\label{thm:GDRO-TINF}
    If $\eta_{\theta, t}$ is nonincreasing,
    \textsc{GDRO-TINF} (Algorithm~\ref{alg:GDRO-TINF}) achieves 
    \begin{align}\label{eq:GDRO-TINF}
        \E[\eps_T(q^*)]
        \leq \frac{1}{T} \left( \frac{G^2}{2}\sum_{t=1}^T \eta_{\theta,t} + \frac{D^2}{2\eta_{\theta, T}} + \sqrt{m}M^2\eta_{q}T + \frac{\sqrt m}{\eta_{q}} \right)
    \end{align}
    for any fixed saddle point $(\theta^*, q^*)$.
    For $\eta_{\theta, t} = \frac{D}{G\sqrt{T}}$ and $\eta_{q} = \frac{1}{M\sqrt T}$, we obtain
\[
    \E[\eps_T(q^*)] \leq \sqrt{2}\frac{\sqrt{G^2D^2 + 4M^2 m}}{\sqrt T}.
\]
\end{theorem}

\subsection{Algorithm for weighted ranking of group losses}\label{subsec:alg-top-k}

We now consider a more general case that $Q$ is a permutahedron.
Applying Algorithm~\ref{alg:general} with the Tsallis entropy regularizer, we obtain the following result.
\begin{theorem}\label{thm:permuta}
    If $\eta_{\theta, t}$ is nonincreasing and $Q$ is a permutahedron, Algorithm~\ref{alg:general} with the Tsallis entropy regularizer achieves the convergence to an approximate-saddle point as Theorem~\ref{thm:GDRO-TINF}.
    Furthermore, the iteration complexity is $O(m \log m + n)$.
\end{theorem}
This implies a convergence rate of $O(\sqrt{\frac{G^2D^2 + M^2 m}{T}})$ for empirical CVaR optimization, which improves $O(\sqrt{\frac{G^2D^2 + M^2 m \log m}{T}})$ convergence by \citet{Curi2020}.
Furthermore, their iteration complexity is $O(m^3)$ due to the $k$-DPP sampling step, so our algorithm is even faster in terms of iteration complexity.

%% file: lower.tex
\section{Lower bound for group DRO}\label{sec:lb}
Theorem~\ref{thm:GDRO-EXP3} states that we can find an $\eps$-optimal solution for group DRO in $O(\frac{G^2D^2 + M^2 m \log m}{\eps^2})$ calls to stochastic oracles.
Next, we show that this query complexity is almost information-theoretically optimal.

Let $\caL$ be a class of convex $G$-Lipschitz loss functions $\ell : \Theta \to [0, M]$.
Given a loss function $\ell \in \caL$, and an $m$-set $\caP = \{P_1, \dots, P_m\}$ of distributions, denote the optimality gap of $\theta \in \Theta$ by
\[
    R(\theta, \ell, \caP) = \max_{P \in \caP} \E_{z \sim P}[\ell(\theta; z)] - \min_{\theta^* \in \Theta}\max_{P \in \caP} \E_{z \sim P}[\ell(\theta^*; z)].
\]
Let $\caA_T$ be the set of algorithms that outputs $\hat\theta \in \Theta$ making $T$ queries to the stochastic oracle.

\begin{theorem}[Lower bound for group DRO]\label{thm:lb}
    \[
        \inf_{\hat\theta\in\caA_T} \sup_{\ell \in \caL, \Theta, \caP} \E_\caP[ R(\hat\theta, \ell, \caP)] \geq \Omega\left( \max\left\{ \frac{GD}{\sqrt T}, {M\sqrt\frac{m}{T}} \right\} \right),
    \]
    where $\Theta$ runs over convex sets with diameter $D$ and $\caP$ over $m$-sets of distributions, and $\E_\caP$ denotes the expectation over outcomes of the stochastic oracle in $\caP$.
\end{theorem}

As $\sqrt{x + y} \leq \sqrt{x} + \sqrt{y} \leq \sqrt{2(x+y)}$ for $x, y \geq 0$, this theorem immediately implies that the minimax convergence rate is $\Omega\left(\sqrt{\frac{G^2D^2 + M^2m}{T}} \right)$, which equals the convergence rate achieved by Algorithm~\ref{alg:GDRO-TINF} up to a constant factor.

\paragraph{Proof Outline.}
It suffices to show two lower bounds $\frac{GD}{\sqrt{T}}$ and $M\sqrt{\frac{m}{T}}$ independently.
The former is a well-known lower bound for stochastic convex optimization~\citep{Agarwal2012}.
To illustrate the latter, we take an algorithmic dependent point of view via the Le cam's method.
For any algorithm in $\mathcal A_T$, we need to construct instances $\mathcal P_0, \mathcal P_1$ such that the total variation distance between the distributions over the query outcomes (they depend on both the behavior of the algorithm and the instance) with respect to $\mathcal P_0$ and $\mathcal P_1$ is small.
On the other hand, the objective function of the two instances must be well-separated, i.e., any fixed $\theta$ is $\delta$ sub-optimal for either $\mathcal P_0$ or $\mathcal P_1$.
So, any algorithm that solves group DRO up to error $\delta$ needs to distinguish two instances $\caP_0$ and $\caP_1$.
This implies a query lower bound because the total variation distance of the outcome distributions of these instances is small.
The challenge is how to construct such instances for the regime of small dimensions of $\theta$, e.g, $n = 1$.
To this end, we carefully construct linear functions for $m$ groups using opposite slopes. Then, based on the behavior of the algorithm, we tweak the noise bias in one of the groups with a positive slope, in a way that any fixed $\theta$ is $\Theta(\delta)$ sub-optimal for one of these instances.
For the detailed proof, see Appendix~\ref{app:lb}.

%% file: experiments.tex
\pgfplotsset{
    cycle list/Dark2,
    cycle multiindex* list={
        Dark2\nextlist
        dashed,solid,densely dashdotted\nextlist
    },
    table/col sep=comma,
    grid=both,
    every axis plot/.append style={ultra thick}
}
\begin{figure*}[t]
    \centering
    \pgfplotsset{xlabel={Iteration $T$},
    ylabel={Optimality gap},
    width=\linewidth,
    legend entries={Sagawa~et~al., GDRO-EXP3P, GDRO-TINF},
    legend columns=-1,
    legend style={font=\small},
    legend to name=named
    }
    \begin{subfigure}[b]{.35\linewidth}
    \begin{tikzpicture}
        \begin{loglogaxis}[]
        \addplot table {figdata/adult_Sagawa_et_al.csv};
        \addplot table {figdata/adult_EXP3P.csv};
        \addplot table {figdata/adult_TINF.csv};
        \end{loglogaxis}
    \end{tikzpicture}
    \caption{Logistic loss}
    \end{subfigure}
    \hspace{3em}
    \begin{subfigure}[b]{.35\linewidth}
    \begin{tikzpicture}
        \begin{loglogaxis}[]
        \addplot table {figdata/adult_hinge_m6_Sagawa_et_al.csv};
        \addplot table {figdata/adult_hinge_m6_EXP3P.csv};
        \addplot table {figdata/adult_hinge_m6_TINF.csv};
        \end{loglogaxis}
    \end{tikzpicture}
    \caption{Hinge loss}
    \end{subfigure}
    \ref{named}
    \caption{Results on Adult dataset for convex losses. Both axes are log-scale.}\label{fig:adult}
\end{figure*}

\section{Experiments}\label{sec:experiments}
In this section, we compare our algorithms with baseline algorithms using real-world datasets in the group DRO setting for both convex and deep learning regimes.
The additional detail of experiments as well as an additional experiment are provided in Appendix~\ref{app:experiments}.
The experiment codes are available in Supplementary materials.

\subsection{Experiment in the convex regime}
First, we validate our convergence analysis for the convex regime. The experiment setup is adopted from \citet{Namkoong2016}.
\paragraph{Dataset.}
We use Adult dataset~\citep{Dua:2019}, which consists of age, gender, race, educational background, and many other attributes of $48,842$ individuals from the US census.
The task is to predict whether the person's income is greater than $50,000$ USD or not.
We set up 6 groups based on the race and gender attributes: each group corresponds to a combination of $\{\text{black}, \text{white}, \text{others} \} \times \{\text{female}, \text{male}\}$.
Converting the categorical features to dummy variables, we obtain a $101$-dimensional feature vector $a \in \R^n$ ($n = 101$) for each individual.
We train the linear model with the logistic loss and hinge loss functions.
The group-DRO objective is the worst empirical loss over the 6 groups $\max_{i = 1}^6 \frac{1}{\abs{I_i}} \sum_{(a, b) \in I_i} \ell(\theta; a, b),$:
where $I_i$ is the set of data points in the $i$th group.
The feasible region is the Euclidean ball of radius $D=10$.


\subsubsection{Algorithms}
We implemented \textsc{GDRO-EXP3P}, \textsc{GDRO-TINF}, and the algorithm in~\citep{Sagawa2020} in Python.
We ran our algorithms for $T=$2,000,000 iterations.


\paragraph{Step sizes.}
The choice of step sizes is crucial to the practical performance of first-order methods.
We found that the decreasing step size $\eta_{\theta, t} \sim 1/\sqrt{t}$ for $\theta_t$ and the fixed step size $\eta_{q} \sim 1/\sqrt{T}$ for $q_t$ gave the best results.
More precisely, we set $\eta_{\theta, t} = \frac{C_{\theta}D}{\sqrt{t}}$ ($t \in [T]$) and $\eta_{q} = C_q\sqrt\frac{\log m}{m T}$, where $C_\theta \in [0.1, 5.0]$ and $C_q \in [0.1, 3.0]$ are hyper-parameters tuned for each algorithm.
We used the best hyper-parameter found by Optuna~\citep{Optuna2019} for the shown results.

\paragraph{Mini-batch and Initialization.}
The use of mini-batch often improves the stability of stochastic gradient algorithms.
In our experiments, we used mini-batches of size $10$ to evaluate stochastic gradients.
Neither the objective values of outputs nor the stability was improved with larger mini-batch sizes.
The group DRO objective is evaluated using the entire dataset. Further, we initialized the algorithms with $\theta_1 = \mathbf{0}$.


\subsubsection{Results}
In Figure~\ref{fig:adult}, we plot the optimality gap of the averaged iterate $\frac{1}{T}\sum_{t=1}^T \theta_t$ against the number of iteration $T$.
We observe that all the algorithms converge with a rate roughly $T^{-0.5}$ for both loss functions, consistent with our convergence bound.
Furthermore, our algorithms (\textsc{GDRO-EXP3P} and \textsc{GDRO-TINF}) achieve faster convergence compared to the algorithm by \citet{Sagawa2020}.
Interestingly, \textsc{GDRO-TINF} achieves a $10^{-4}$ optimality gap in $T = 10^6$ iterations, which is faster than the theoretical $T^{-0.5}$ rate in Theorem~\ref{thm:GDRO-TINF}.
\par We perform additional experiments in the deep learning regime across five benchmark datasets from WILDS ~\citep{Sagawa2020} including Waterbirds, FMOW, MultiNLI, etc. Worst group and average test performance of various methods is reported in Appendix \ref{sec:dlexp}

\section{Conclusion}
In this work we settle the optimal achievable regret in the group DRO problem, up to a log factor, by (1) developing a new technique that enables us to employ online optimization techniques in offline robust optimization, and (2) combining the right ingredients from online adversarial algorithms to achieve the almost best rate for group DRO. We hope that our work further encourages researchers in the future to employ such reductions from online to offline optimization. Besides the demonstrated theoretical guarantees, our extensive experiments on real and synthetic data illustrate that our algorithm is competitive with state-of-the-art methods.

%% file: appendix.tex
\section{Preliminaries of online convex optimization and no-regret dynamics}\label{app:preliminaries}
In this section, we briefly introduce necessary results from online convex optimization (OCO). For the further details of OCO, refer to \cite{Hazan2016book,Orabona2019book}.

\subsection{Regret Bounds of OCO algorithms}
Let $X \subseteq \R^d$ be a compact convex set and $\Psi: X \to \R$ be a strictly convex function such that $\norm{\partial\Psi(x)} \to +\infty$ as $x \to \partial X$.
Online mirror descent (OMD) is the following online learning algorithm.
For $t = 1, \dots, T$:
\begin{enumerate}
    \item Let $\tilde x_{t+1} \in \R^n$ be the solution of $\nabla\Psi(\tilde x_{t+1}) = \nabla\Psi(x_t) - \eta_t \nabla_t$, where $\eta_t > 0$ is a step size and $\nabla_t = \nabla f_t(x_t)$ is the gradient feedback of round $t$.
    \item Let $x_{t+1} \in \argmin_{x \in X} D_\Psi(x, \tilde x_{t+1})$, where $D_\Psi(x, y) = \Psi(x) - \Psi(y) - \nabla\Psi(y)^\top(x - y)$ is the Bregman divergence with respect to $\Psi$.
\end{enumerate}

We use the following regret bound.
\begin{lemma}[Regret Bound of OMD; see, e.g., {\citet[Theorem~6.8]{Orabona2019book}}]\label{lem:OMD-regret}
    OMD satisfies that for any $x^* \in X$,
    \begin{align}
        \sum_{t=1}^T f_t(x_t) - \sum_{t=1}^T f_t(x^*)
        \leq \frac{1}{2}\sum_{t=1}^T \eta_t \norm{\nabla_t}_{t,*}^{2} + \frac{D_\Psi(x^*, x_1)}{\eta_1} + \sum_{t=2}^T \left( \frac{1}{\eta_{t}} - \frac{1}{\eta_{t-1}} \right)D(x^*, x_t),
    \end{align}
    where $\norm{x}_t$ denotes the local norm, i.e., $\norm{x}_t := \sqrt{x^\top \nabla^2 \Psi(z_t) x}$ for some $z_t \in [x_t, \tilde x_{t+1}]$ and $\norm{x}_{t,*} := \sqrt{x^\top \nabla^2 \Psi(z_t)^{-1} x}$ is its dual norm.
\end{lemma}

In this paper, we use regret bounds for the following specific choices of $\Psi$.

\paragraph{Online Gradient Descent} OMD for $\Psi(x) = \frac{1}{2}\norm{x}_2^2$ on a generic compact convex set $X$ is simply online gradient descent (OGD)~\cite{Zinkevich2003}:
\begin{align*}
    x_{t+1} = \proj_{X}(x_t - \eta_t \nabla_t).
\end{align*}
Note that $D(x, y) = \frac{1}{2}\norm{x - y}_2^2$ and the minimizing the Bregman divergence is given by orthogonal projection.
\begin{lemma}[Regret Bound of OGD]\label{lem:OGD-regret}
    OMD satisfies that for any $x^* \in X$,
    \begin{align}
        \sum_{t=1}^T f_t(x_t) - \sum_{t=1}^T f_t(x^*)
        \leq \frac{1}{2}\sum_{t=1}^T \eta_t \norm{\nabla_t}^2 + \frac{\norm{x^* - x_1}_2^2}{2\eta_1} + \frac{1}{2}\sum_{t=2}^T \left( \frac{1}{\eta_{t}} - \frac{1}{\eta_{t-1}} \right) \norm{x^* -  x_t}_2^2.
    \end{align}
    If we use decreasing step sizes and $\max_{t=1}^T \norm{x^* - x_t} \leq D$, we have
    \begin{align}
        \sum_{t=1}^T f_t(x_t) - \sum_{t=1}^T f_t(x^*)
        \leq \frac{1}{2}\sum_{t=1}^T \eta_t \norm{\nabla_t}^2 + \frac{D^2}{2\eta_T}.
    \end{align}
\end{lemma}

\paragraph{Hedge}
OMD for $\Psi(x) = \sum_i (x_i \log x_i - x_i)$ on the probability simplex is the Hedge algorithm.
\begin{align*}
    \tilde x_{t+1} = x_t \exp(-\eta_t \nabla_t), \quad
    x_{t+1} = \frac{\tilde x_{t+1}}{\norm{\tilde x_{t+1}}_1}.
\end{align*}
Note that $\nabla^2 \Psi(x) = \diag(1/x_i)$.
If $\nabla_t \geq 0$, then $\tilde x_{t+1} \leq x_t$ and $\norm{\nabla_t}_{t,*} \leq \norm{\nabla_t}_{\nabla^2\Psi(x_t)^{-1}}$.
For $x_1 = \bfone/d$, $D(x^*, x_1) \leq \log d$ for any $x^*$.
\begin{lemma}[Regret Bound of Hedge]\label{lem:Hedge-regret}
    For $\nabla_t \geq 0$ ($t=1, \dots, 1$), Hedge with fixed step size $\eta > 0$ satisfies
    \begin{align}
        \sum_{t=1}^T f_t(x_t) - \sum_{t=1}^T f_t(x^*)
        \leq \frac{\eta}{2}\sum_{t=1}^T \norm{\nabla_t}_{\nabla^2\Psi(x_t)^{-1}}^{2} + \frac{\log d}{\eta}.
    \end{align}
\end{lemma}

\paragraph{Tsallis-INF}
OMD for $\Psi(x) = 2(1 - \sum_i \sqrt{x_i})$ on the probability simplex is the Tsallis-INF algorithm:
\begin{align*}
    \tilde x_{t+1} \gets x_t \left(\bfone - \eta_t\nabla_t \right)^{-2}, \quad
    x_{t+1} = \left(\frac{1}{\sqrt{\tilde x_{t+1}}} - \alpha \bfone \right)^{-2},
\end{align*}
where $\alpha$ is the scaling factor such that $x_{t+1}$ is in the probability simplex.
Note that if $\nabla_t \geq 0$, then $\tilde x_{t+1} \leq x_t$ and $\norm{\nabla_t}_{t,*} \leq \norm{\nabla_t}_{\nabla^2\Psi(x_t)^{-1}}$ as in Hedge.
For $x_1 = \bfone/d$, $D(x^*, x_1) \leq \sqrt d$ for any $x^*$.
\begin{lemma}[Regret Bound of Tsallis-INF]\label{lem:TINF-regret}
    For $\nabla_t \geq 0$ ($t=1, \dots, 1$), Tsallis-INF with fixed step size $\eta > 0$ satisfies
    \begin{align}
        \sum_{t=1}^T f_t(x_t) - \sum_{t=1}^T f_t(x^*)
        \leq \frac{\eta}{2}\sum_{t=1}^T \norm{\nabla_t}_{\nabla^2\Psi(x_t)^{-1}}^{2} + \frac{\sqrt d}{\eta}.
    \end{align}
\end{lemma}

\subsection{Convergence of No-Regret Dynamics}
Let us write DRO~\eqref{eq:general-DRO} as
\begin{align*}
    \min_{\theta \in \Theta}\max_{q \in Q} L(\theta, q).
\end{align*}
Note that $L(\theta, q)$ is convex in $\theta$ and linear in $q$.

Let us assume that we apply stochastic no-regret dynamics to this minimax problem.
The $\theta$-player and $q$-player run online algorithms on $\Theta$ and $Q$, respectively.
The feedback to $\theta$-player and $q$-player are $\hat\nabla_{\theta, t}$ and $\hat\nabla_{q, t}$, respectively, which are unbiased gradient estimators of $L$.
We can analyze the optimality gap of stochastic no-regret dynamics using the regrets.
Let $\theta^*$ be an optimal solution and 
\[
    \eps_T := \max_{q \in Q} L(\bar\theta_{1:T}, q) - \max_{q \in Q} L(\theta^*, q).
\]
be the optimality gap of the averaged iterate $\bar\theta_{1:T} = \frac{1}{T} \sum_{t=1}^T \theta_t$.
Let
\begin{align*}
    R_\theta(T; \theta^*) &= \sum_{t=1}^T L(\theta_t, q_t) -  \sum_{t=1}^T L(\theta^*, q_t)\\
    R_q(T) &= \max_{q \in \Delta_m} \sum_{t=1}^T L(\theta_t, q) - \sum_{t=1}^T L(\theta_t, q_t)
\end{align*}
be regrets of the $\theta$-player and $q$-player, respectively.
Then, by the definition of regret and Jensen's inequality, we have
\begin{align*}
    \eps_T
    &\leq \max_{q \in Q}\frac{1}{T}\sum_{t=1}^T  L(\theta_t, q) - \max_{q \in Q} L(\theta^*, q) \\
    &= \frac{R_q(T)}{T} + \frac{1}{T}\sum_{t=1}^T  L(\theta_t, q_t) - \max_{q \in Q} L(\theta^*, q) \\
    &\leq \frac{R_q(T)}{T} + \frac{1}{T}\sum_{t=1}^T  L(\theta_t, q_t) - \frac{1}{T}\sum_{t=1}^T L(\theta^*, q_t) \\
    &= \frac{R_q(T) + R_\theta(T; \theta^*)}{T}.
\end{align*}
Therefore,
\begin{align}\label{eq:regret-to-convergence}
    \E[\eps_T] \leq \frac{\E[R_q(T) + R_\theta(T; \theta^*)]}{T},
\end{align}
where the expectation is taken over the randomness of gradient estimators and the algorithm.
This proves \eqref{eq:converge-regret}.

Similarly, let 
\[
    R_q(T; q^*) =  \sum_{t=1}^T L(\theta_t, q^*) - \sum_{t=1}^T L(\theta_t, q_t)
\]
be the regret of the $q$-player with respect to $q^* \in Q$.
Then, for any fixed saddle point $(\theta^*, q^*)$, 
\begin{align*}
    \eps_T(q^*)
    &\leq \frac{1}{T}\sum_{t=1}^T  L(\theta_t, q^*) - L(\theta^*, q^*) \\
    &= \frac{R_q(T;q^*)}{T} + \frac{1}{T}\sum_{t=1}^T  L(\theta_t, q_t) - L(\theta^*, q^*) \\
    &\leq \frac{R_q(T;q^*)}{T} + \frac{1}{T}\sum_{t=1}^T  L(\theta_t, q_t) - \frac{1}{T}\sum_{t=1}^T L(\theta^*, q_t) \tag{since $q^* \in \argmax_{q\in Q} L(\theta^*, q)$} \\
    &= \frac{R_q(T;q^*) + R_\theta(T; \theta^*)}{T}.
\end{align*}
Therefore,
\begin{align}\label{eq:regret-to-saddle}
    \E[\eps_T(q^*)] \leq \frac{\E[R_q(T; q^*) + R_\theta(T; \theta^*)]}{T}.
\end{align}

\subsection{EXP3P}
To analyze the convergence rate using the above bound, we need to bound the expected regret of the $q$-player for an \emph{adaptive} adversary.
This is not possible by OMD bounds because it only considers fixed optimal solutions, i.e., an \emph{oblivious} adversary.
Thankfully, in group DRO, we can use the EXP3P algorithm~\cite{Auer2003}, which has desired regret bounds for adaptive adversaries.

\begin{algorithm}
\caption{\textsc{EXP3P}}\label{alg:EXP3P}
    \begin{algorithmic}[1]
        \REQUIRE parameters $\beta, \eta, \gamma > 0$
        \STATE Let $q_1 = (1/m, \dots, 1/m)$.
        \FOR{$t = 1, \dots, T$}
            \STATE Sample $i_t \sim q_t$.
            \STATE Let $\tilde g_t := \frac{\ell_{i_t}\bfe_{i_t} + \beta\bfone}{q_{i_t}}$, $G_{t} := \sum_{\tau=1}^t \tilde g_t$, and $Z = \sum_{i \in [m]} \exp(\eta G_{t,i})$.
            \STATE $q_{t+1} \gets (1-\gamma)\frac{\exp(\eta G_t)}{Z} + \frac{\gamma\bfone}{m}$.
        \ENDFOR
    \end{algorithmic}
\end{algorithm}

\begin{theorem}[{see, e.g., \cite[Theorem~3.4]{Bubeck2012}}]\label{thm:EXP3P}
    Let $g_t \in [0,1]^m$ for $t \in T$.
    For $\beta = \sqrt{\frac{\log m}{mT}}$, $\eta = O(\sqrt{\frac{\log m}{mT}})$, and $\gamma = O(\sqrt{\frac{m\log m}{T}})$, EXP3P achieves
    \[
        \E[R_q(T)] = \E\left[ \max_{i^* \in [m]} \sum_{t=1}^T (g_{t,i^*} - g_{t,i_t}) \right] \lesssim \sqrt{mT \log m}. 
    \]
\end{theorem}

\section{Ommited Proofs}\label{app:proofs}
\subsection{Proof of Theorem~\ref{thm:general}}
Let $I_t$ and $z_t$ be the chosen group and the sample at iteration $t$, respectively.
Observe that Algorithm~\ref{alg:general} is stochastic no-regret dynamics with OGD, OMD, and gradient estimators
\begin{align*}
    \hat\nabla_{\theta, t} := \nabla_\theta \ell(\theta_t; z_t), \quad
    \hat\nabla_{q, t} := \frac{1}{q_{t,I_t}} \ell(\theta_t; z_t)\bfe_{I_t}.
\end{align*}
For OGD, we use Lemma~\ref{lem:OGD-regret}.
We have $\norm{\hat\nabla_{\theta,t}}_2^2 \leq G$ by assumption.
Therefore,
\[
\E[R_\theta(T)] \leq \frac{G^2}{2}\sum_{t=1}^T \eta_{\theta,t} + \frac{D^2}{2\eta_{\theta, T}}
\]
by Lemma~\ref{lem:OGD-regret}.
For OMD, we use Lemma~\ref{lem:OMD-regret}.
Since $\hat\nabla_{q,t} = \frac{1}{q_{t,I_t}} \ell(\theta_t, z_t) \bfe_{I_t}$, we obtain
\begin{align*}
    \norm{\hat\nabla_{q,t}}_{\nabla^2\Psi(q_t)^{-1}}^2
    &= \frac{\ell(\theta_t, z_t)^2 (\nabla^2\Psi(q_t)^{-1})_{I_t, I_t}}{q_{t, I_t}^2} \\
    &\leq \frac{M^2 (\nabla^2\Psi(q_t)^{-1})_{I_t, I_t}}{q_{t, I_t}^2}.
\end{align*}
Hence we obtain from Lemma~\ref{lem:OMD-regret},
\begin{align*}
    \E[R_q(T; q^*)]
    &\leq \frac{1}{2}\sum_{t=1}^T \eta_{q}
    \E_{I_t} \left[ \frac{(\nabla^2\Psi(q_t))^{-1}_{I_t,I_t}}{q_{t,I_t}^2} \right]
    + \frac{D_\Psi(q^*, q_1)}{\eta_{q}}
\end{align*}
for any saddle point $(\theta^*, q^*)$.
Now the theorem is immediate from \eqref{eq:regret-to-saddle}.

\subsection{Proof of Theorem~\ref{thm:GDRO-EXP3}}
Observe that Algorithm~\ref{alg:GDRO-EXP3} is stochastic no-regret dynamics with OGD, EXP3P, and the same gradient estimators as above.
Without loss of generality, we can assume $M=1$; general case follows scaling the loss functions accordingly.
Using the regret bound in Theorem~\ref{thm:EXP3P}, we have
\[
    \E[\eps_q(T)] \lesssim \sqrt{mT \log m}.
\]
Now the theorem follows from \eqref{eq:regret-to-convergence}.




\subsection{Proof of Theorem~\ref{thm:GDRO-TINF}}
Observe that Algorithm~\ref{alg:GDRO-TINF} is stochastic no-regret dynamics with OGD, Tsallis-INF, and the same gradient estimators as above.
From Theorem~\ref{thm:general}, it suffices to bound the local norm with respect to the Tsallis entropy regularizer.
Observe that $\nabla^2\Psi(q_t) = \frac{1}{2}\diag(q_t^{-2/3})$.
Conditioned on $I_1, \dots, I_{t-1}$, we have
\begin{align*}
    \E_{I_t} \left[ \frac{(\nabla^2\Psi(q_t))^{-1}_{I_t,I_t}}{q_{t,I_t}^2} \right]
    &= \sum_{i=1}^m \Pr(I_t = i) \cdot \frac{2}{q_{t, i}^{1/2}} \\
    &\leq 2 \sum_{i=1}^m q_{t, i}^{1/2} \\
    &\leq 2 \sqrt{m}\sqrt{\sum_{i=1}^m q_{t,i}} \\
    &= 2 \sqrt{m}.
\end{align*}
By Theorem~\ref{thm:general}, we obtain \eqref{eq:GDRO-TINF}.

\section{Ommited Proofs in Section~\ref{sec:lb}}\label{app:lb}
In this section, we prove Theorem~\ref{thm:lb}.

We show that the minimax optimality gap is $\Omega(GD/\sqrt{T})$ and $\Omega(M\sqrt{m/T})$, separately.
The first lower bound is immediate from the well-known lower bound of stochastic convex optimization (see, e.g., \cite{Agarwal2012}).
Hence, it suffices to show the second lower bound.

Note that it suffices to show the lower bound for a constant $M$; below we construct instances with $M=2$.
The general case follows by scaling the objective with $M$.
Consider the following instance of group DRO which we construct with respect to an $m$-dimensional vector $\mu = (\mu_1, \dots, \mu_m) \in [0,1]^m$ of Bernoulli biases.
Let $\Theta$ be the unit interval $[0,1]$.
Let
\[
    \ell(\theta; Z) = Z_1 f_1(\theta) + Z_2 f_2(\theta) + Z_3,
\]
where $f_1(\theta) = \delta \theta$ and $f_2(\theta) = \delta(1-\theta)$ are linear functions over the interval $[0,1]$ and $\delta > 0$ is the accuracy parameter determined later.
We define a joint distribution $P_i$ of $Z$ as follows: for $i = 1, \dots, m-1$, let
\begin{align*}
    P_i :
    \begin{cases}
        Z_1 = 0 \ \ \text{a.s.}\\
        Z_2 = 1 \ \ \text{a.s.}\\
        Z_3 \sim \text{Ber}(\mu_i)
    \end{cases}
\end{align*}
where $a.s.$ stands for almost surely. For the last group distribution $i = m$, let
\begin{align*}
    P_m :
    \begin{cases}
        Z_1 = 1 \ \ \text{a.s.}\\
        Z_2 = 0 \ \ \text{a.s.}\\
        Z_3 \sim \text{Ber}(\mu_m).
    \end{cases}
\end{align*}
Then,
\[
    \E_{Z \sim P_i}[\ell(\theta; Z)] =
    \begin{cases}
        \delta(1-\theta) + \mu_i & (i=1,\dots,m-1)\\
        \delta \theta + \mu_m.   & (i=m)
    \end{cases}
\]
The information of an outcome of a single stochastic oracle call to $P_i$ is no more than that of a single sample of the $i$th Bernoulli distribution $\mathrm{Ber}(\mu_i)$.

Let us fix $\hat\theta \in \caA_T$ arbitrarily.
Let $\caP_0$ be the set of distributions $(P_i)$ constructed as above with
\[
    \mu^0 = (1/2, 1/2, \dots, 1/2).
\]
It is clear that $\min_{\theta^* \in \Theta}\max_{P \in \caP_0} \E_{Z \sim P}[\ell(\theta^*; Z)] = 1/2 + \delta/2$, which is attained by $\theta^* = 1/2$.
We denote by $Q_0$ the distribution of the outcomes of stochastic oracles observed by $\hat\theta$ under $\caP_0$.
Furthermore, let $T_i$ be the expected number of queries to the $i$th stochastic oracle made by $\hat\theta$ under $\caP_0$.
Since $\hat\theta$ makes $T$ queries in total, there exists $i^* \neq m$ such that $T_{i^*} \leq \frac{T}{m-1}$.
Let $\caP_1$ be the set of distributions constructed as above with
\[
    \mu^1 = (1/2, 1/2, \dots, 1/2, \stackrel{i^*}{1/2 + \delta}, 1/2, \dots, 1/2).
\]

\begin{lemma}\label{lem:separation-R}
    \[
    \max\{ R(\theta, \ell, \caP_0), R(\theta, \ell, \caP_1) \} \geq \delta/4
    \]
for any $\theta$.
\end{lemma}
\begin{proof}
We consider two different cases: $\theta \geq \frac{3}{4}$ and $\theta < \frac{3}{4}$.

For $\theta \geq 3/4$, we have $R(\theta, \ell, \caP_0) \geq \frac{\delta}{4}$ since
\begin{align*}
    \max_{P \in \caP_0} \E_{Z \sim P}[\ell(\theta^*; Z)] = \max_{P \in \caP_0} \E_{Z \sim P}[\ell(1/2; Z)] = \delta/2 + 1/2,
\end{align*}
while
\begin{align*}
    \max_{P \in \caP_0} \E_{Z \sim P}[\ell(\theta; Z)] =  \E_{Z \sim P_m}[\ell(\theta; Z)] \geq \delta/2 + \delta/4 + 1/2.
\end{align*}
For the other case, $\theta < 3/4$, we show that $R(\theta, \ell, \caP_1) \geq \frac{\delta}{4}$. This holds as
\begin{align*}
    \max_{P \in \caP_1} \E_{Z \sim P}[\ell(\theta^*; Z)] = \max_{P \in \caP_1} \E_{Z \sim P}[\ell(1; Z)] = \delta + 1/2,
\end{align*}
while
\begin{align*}
    \max_{P \in \caP_1} \E_{Z \sim P}[\ell(\theta; Z)] =  \E_{Z \sim P_{i^*}}[\ell(\theta; Z)] \geq \delta + \delta/4 + 1/2.
\end{align*}
This completes the proof.
\end{proof}

We denote by $Q_1$ the distribution of the outcomes of stochastic oracles observed by $\hat\theta$ under $\caP_1$. By Lecam's two-point method,
    \[
        \inf_{\hat\theta}\sup_{\caP} \E[R(\hat\theta, \ell, \caP)]
        \geq \frac{\delta}{2}\left( 1 - \dTV(Q_0,Q_1) \right),
    \]
    where the expectation is taken over the outcomes of the stochastic oracle and $\dTV$ denotes the total variation distance.
We proceed to bound the right-hand side.
By the Pinsker inequality,
\[
    \dTV(Q_0,Q_1)^2 \lesssim \DKL(Q_0 \mid\mid Q_1),
\]
where $\DKL$ denotes the Kullback-Leibler divergence.
By the standard computation, we can show the following.
\begin{lemma}\label{lem:KL}
\begin{align*}
    \DKL(Q_0 \mid\mid Q_1) \lesssim \delta^2 T_{i^*}  \leq \frac{\delta^2 T}{m-1}.
\end{align*}
    for $\delta \in (0, 1/4)$,
\end{lemma}
Thus, setting $\delta = O(\sqrt{m/T})$, we obtain
\[
    \inf_{\hat\theta\in\caA_T} \sup_{\ell \in \caL, \caP} R(\hat\theta, \ell, \caP) \gtrsim \sqrt{\frac{m}{T}},
\]
which completes the proof of Theorem~\ref{thm:lb}.


\subsection{Proof of Lemma~\ref{lem:KL}}
Now we prove Lemma~\ref{lem:KL} for the completeness.
Let $o_t$ be the outcome of the $t$th query to the stochastic oracle.
We will use the shorthand notation $o_{1:t}$ to denote the outcomes $(o_1, \dots, o_t)$ up to the $t$th queries.
Let $I_t \in [m]$ be the index of stochastic oracles that $\hat\theta$ queries in the $t$th round.
Note that $I_t$ is determined by $o_{1:t-1}$.
Then, we have
\begin{align*}
   \DKL(Q_0 \mid\mid Q_1)
   &= \sum_{t = 1}^T \DKL(Q_0(o_t \mid o_{1:t-1}) \mid\mid Q_1(o_t \mid o_{1:t-1})) \tag{chain rule}\\
   &\leq \sum_{t = 1}^T \E_{o_{1:t-1} \sim Q_0}\left[\DKL(\Ber(\mu^0_{I_t}) \mid\mid \Ber(\mu^1_{I_t})) \right] \tag{data-processing inequality} \\
   &= \sum_{t = 1}^T \E_{o_{1:t-1} \sim Q_0}\left[\bfone[I_t = i^*] \DKL(\Ber(1/2) \mid\mid \Ber(1/2 + \delta)) \right] \\
   &= T_{i^*} \cdot \DKL(\Ber(1/2) \mid\mid \Ber(1/2+\delta)).
\end{align*}
Furthermore, for $\delta \in (0, 1/4)$,
\begin{align*}
    \DKL(\Ber(1/2) \mid\mid \Ber(1/2+\delta))
    &= \frac{1}{2} \log \frac{1/2}{1/2 + \delta} + \frac{1}{2}\log \frac{1/2}{1/2 - \delta} \\
    &= \frac{1}{2} \log \left(1 - \frac{2\delta}{1 + 2\delta}\right) + \frac{1}{2}\log \left(1 + \frac{2\delta}{1 - 2\delta} \right) \\
    &\leq - \frac{\delta}{1 + 2\delta} + \frac{\delta}{1 - 2\delta} \\
    &\leq \frac{4\delta^2}{(1+2\delta)(1-2\delta)} \leq 8\delta^2.
\end{align*}
This completes the proof.

\section{Algorithm of Sagawa~et~al. for group DRO}\label{app:sagawa}
Here we present the algorithm by~\cite{Sagawa2020} for group DRO.
Algorithm~\ref{alg:Sagawa-et-al} shows the pseudocode.
In each iteration $t$, the algorithm picks group index $i_t \in [m]$ uniformly at random and obtains an i.i.d.~sample $z \sim P_{i_t}$.
Then, the algorithm performs one step of projected gradient descent and Hedge on $\theta_t \in \Theta$ and $q_t \in \Delta_m$, respectively, where the gradients are estimated with $i_t$ and $z$.
Note that $q_t$ is only used for the scaling factor of the gradient estimator.
In each iteration, the algorithm performs a single orthogonal projection onto $\Theta$ and $O(m+n)$ operations to update $\theta_t, q_t$.

\begin{algorithm}
    \caption{Algorithm of Sagawa~et~al.}\label{alg:Sagawa-et-al}
    \begin{algorithmic}[1]
        \REQUIRE initial solution $\theta_1 \in \Theta$, number of iteration $T$, and step sizes $\eta_{\theta, t} > 0$ ($t \in [T]$), $\eta_q > 0$.
        \STATE Let $q_t = (1/m, \dots, 1/m)$.
        \FOR{$t = 1, \dots, T$}
            \STATE Sample $i_t \sim [m]$ uniformly at random.
            \STATE Call the stochastic oracle to obtain $z \sim P_{i_t}$.
            \STATE $\theta_{t+1} \gets \proj_\Theta(\theta_t - mq_{t,i_t}\eta_{\theta, t}\nabla_\theta\ell(\theta_t; z))$
            \STATE $\tilde q_{t+1} \gets q_t\exp(m\eta_{q}\ell(\theta_t; z)\bfe_{i_t})$ and $q_{t+1} \gets \frac{\tilde q_{t+1}}{\sum_i \tilde q_{t+1, i}}$.
        \ENDFOR
        \RETURN $\frac{1}{T}\sum_{t=1}^T \theta_t$.
    \end{algorithmic}
\end{algorithm}

In the view of no-regret dynamics, the main difference between our algorithms and \citet{Sagawa2020} is the gradient estimators; see Table~\ref{tab:no-regret}.

\newcommand{\gradcellwidth}{10em}
\newcommand{\Sagawagradest}{%
\parbox{\gradcellwidth}{%
\begin{alignat*}{3}
 \hat\nabla_{\theta, t} &:= m q_{t,i} \nabla_\theta \ell(\theta; z), &\quad
 \hat\nabla_{q, t} &:= m \ell(\theta_t; z)\bfe_i. &\quad (i \sim [m], z \sim P_i)
\end{alignat*}
}}
\newcommand{\ourgradest}{%
\parbox{\gradcellwidth}{%
\begin{alignat*}{3}
    \hat\nabla_{\theta, t} &:= \nabla_\theta \ell(\theta_t; z), &\quad
    \hat\nabla_{q, t} &:= \frac{1}{q_{t,i}} \ell(\theta_t; z)\bfe_i. & \quad (i \sim q_t, z \sim P_i)
\end{alignat*}%
}}
\begin{table*}[h]
    \caption{Algorithms as stochastic no-regret dynamics}\label{tab:no-regret}
    \centering
    \setlength{\abovedisplayskip}{0pt}
    \small
    \begin{tabular}{c|ccc}
        & $\caA_\theta$ & $\caA_q$ & $\hat\nabla_{\theta,t}, \hat\nabla_{q,t}$ \\ \hline
      Algorithm~\ref{alg:Sagawa-et-al} & OGD & Hedge &  \Sagawagradest \\
      Algorithm~\ref{alg:GDRO-EXP3} & OGD & EXP3P &  \ourgradest \\
      Algorithm~\ref{alg:GDRO-TINF} & OGD & Tsallis-INF & \ourgradest
    \end{tabular}
\end{table*}

\section{Additional experiments}\label{app:experiments}
\subsection{Experiment with synthetic dataset for convex regime}

\paragraph{Dataset.}
To observe the performance of the algorithms over the regime of high-dimension model parameters and the larger number of groups, we also conducted experiments using the following synthetic instances.
First, we set $n=500$ and varied $m \in \{10, 50, 100\}$.
For each group $i \in [m]$, we generated the true classifier $\theta^*_i \in \R^n$ from the uniform distribution over the unit sphere in $\R^n$.
The $i$th group distribution $P_i$ was the empirical distribution of 1,000 data points,  where each data point $(a, b)$ was drawn as $a \sim N(0, I_n)$ and $b = \sign(a^\top \theta^*_i)$ with probability $0.9$ and $b = -\sign(a^\top \theta^*_i)$ with probability $0.1$.
We trained the linear model with the hinge loss function.
Finally, the group-DRO objective is
\[
    \max_{i = 1}^m  \E_{(a, b) \sim P_i} [\ell(\theta; a, b)].
\]
The feasible region is the Euclidean ball of radius $D=10$.

\begin{figure*}[h]
    \pgfplotsset{
    xlabel={Iteration $T$},
    legend style={font=\small}
    }
    \centering
    \vspace{-1em}
    \begin{tikzpicture}
        \begin{groupplot}[group style={group size=3 by 1, ylabels at=edge left, horizontal sep=4em},
            width=.34\linewidth,
            ylabel=Objective,
            legend columns=-1,
            legend to name=named2]
            \nextgroupplot[title={$m=10$}, xmode=log,ymode=log]
                \addplot table {figdata/synthetic_m10_Sagawa_et_al.csv};
                \addplot table {figdata/synthetic_m10_EXP3P.csv};
                \addplot table {figdata/synthetic_m10_TINF.csv};
            \nextgroupplot[title={$m=50$}, xmode=log,ymode=log]
                \addplot table {figdata/synthetic_m50_Sagawa_et_al.csv};
                \addplot table {figdata/synthetic_m50_EXP3P.csv};
                \addplot table {figdata/synthetic_m50_TINF.csv};
            \nextgroupplot[title={$m=100$}, xmode=log,ymode=log]
                \addplot table {figdata/synthetic_m100_Sagawa_et_al.csv};
                \addplot table {figdata/synthetic_m100_EXP3P.csv};
                \addplot table {figdata/synthetic_m100_TINF.csv};
            \legend{Sagawa~et~al., GDRO-EXP3, GDRO-TINF},
        \end{groupplot}
    \end{tikzpicture}
    \ref{named2}
    \caption{Results on the synthetic dataset for the convex regime. Both axes are log-scale}\label{fig:synthetic}
\end{figure*}

\paragraph{Result.}
In Figure~\ref{fig:synthetic}, we plot the objective values of the averaged iterate against the number of iterations.
For all the values of $m$, our algorithms (especially \textsc{GDRO-EXP3}) consistently achieve smaller loss values faster than the known algorithm.
The performance gap between our algorithms and the known algorithm increased as $m$ grows, which verifies that our algorithms have better dependence on $m$ in the convergence rate.

\subsection{Experiments in the deep learning regime}\label{sec:dlexp}
Our convergence analysis focuses on the convex regime.
However, algorithms designed for the convex regime often work well even for the deep learning regime.
Here, we compare our algorithms with the known algorithms in the deep learning regime.

\begin{table*}[!htb]
    \caption{Worst-group test performance of algorithms for Wilds datasets.  S.~et~al., \textsc{EXP}, and \textsc{TINF} denote the algorithm of \citet{Sagawa2020}, \textsc{GDRO-EXP}, and \textsc{GDRO-TINF}, respectively. The value format is mean $\pm$ standard deviation. The best mean in each row is in bold.\label{tab:wilds}}
    \centering
    \begin{adjustbox}{max width=1.\textwidth}
    \begin{NiceTabular}{@{}ll@{}l|c|c|c|c@{}}
    \CodeBefore
        \rowcolor{white}{1-2}
        \rowcolors[gray]{3}{0.8}{}[cols=3-7,restart]
    \Body
    \toprule
    Dataset & Metric & Heuristics & ERM & S.~et~al. & \textsc{EXP} & \textsc{TINF}  \\
    \midrule
    \Block{3-1}{Waterbirds} & \Block{3-1}{accuracy}
    & standard training & 58.9  $\pm$ 3.21 & 75.5 $\pm$ 1.66 & 70.6 $\pm$ 4.07 & \bf 77.4 $\pm$ 2.31 \\
    & & penalty &  24.3  $\pm$ 2.43 & \bf 84.9 $\pm$ 1.19 & 77.4 $\pm$ 0.72 & 79.0 $\pm$ 0.71 \\
    & & early stop + penalty &  9.2 $\pm$ 1.34 & 86.2 $\pm$ 0.96 & 85.4 $\pm$ 1.44 & \bf 87.5 $\pm$ 2.68 \\
    \midrule
    \Block{2-1}{MultiNLI} & \Block{2-1}{accuracy}
    & standard training &  65.1 $\pm$ 1.22 & 66.4 $\pm$ 1.84 & \bf 72.4 $\pm$ 0.83 & 72.2 $\pm$ 2.47 \\
    & & early stop + penalty & 66.8 $\pm$ 1.56 & 77.7 $\pm$ 3.69 & 76.9 $\pm$ 2.25 & \bf 78.1 $\pm$ 1.63
 \\
    \midrule
    CIVIL-Comments & accuracy
    & standard training & 58.4	$\pm$ 1.29  & 67.5 $\pm$ 3.43 & 62.5 $\pm$ 1.47 & \bf 68.8 $\pm$ 1.34 \\
    \midrule
    FMOW-Wilds & accuracy
    & standard training & \bf 33.4 $\pm$ 0.53 & 31.9 $\pm$ 1.22 & 31.7	$\pm$ 1.21  & 33.31 $\pm$ 0.83
 \\
    \midrule
    PovertyMAP-Wilds & U/R Pearson R 
    & standard training & 0.413 $\pm$ 0.025 & 0.422 $\pm$ 0.031 &  \bf 0.491	$\pm$ 0.037 & 0.439 $\pm$ 0.046 \\
    \bottomrule
    \end{NiceTabular}
    \end{adjustbox}
\end{table*}

\begin{table*}[!htb]
    \caption{Average test performance of algorithms for Wilds datasets.  S.~et~al., \textsc{EXP}, and \textsc{TINF} denote the algorithm of \citet{Sagawa2020}, \textsc{GDRO-EXP}, and \textsc{GDRO-TINF}, respectively. The value format is mean $\pm$ standard deviation. The best mean in each row is in bold.\label{tab:wilds_avgacc}}
    \centering
    \begin{adjustbox}{max width=1.\textwidth}
    \begin{NiceTabular}{@{}cc@{}c|c|c|c|c@{}}
    \CodeBefore
        \rowcolor{white}{1-2}
        \rowcolors[gray]{3}{0.8}{}[cols=3-7,restart]
    \Body
    \toprule
    Dataset & Metric & Heuristics & ERM & S.~et~al. & \textsc{EXP} & \textsc{TINF}  \\
    \midrule
    \Block{3-1}{Waterbirds} & \Block{3-1}{accuracy}
    & standard training & \bf 97.2  $\pm$ 0.35 & 90.5 $\pm$ 0.21 & 91.3 $\pm$ 0.23 & 91.1 $\pm$ 0.52 \\
    & & penalty &  \bf 94.5  $\pm$ 0.78 & 92.2 $\pm$ 0.42 & 90.0 $\pm$ 0.40 & 90.9 $\pm$ 1.50 \\
    & & early stop + penalty &  \bf 92.4 $\pm$ 0.18 & 90.0 $\pm$ 0.70 & 91.4 $\pm$ 0.74 & 90.6 $\pm$ 1.42 \\
    \midrule
    \Block{2-1}{MultiNLI} & \Block{2-1}{accuracy}
    & standard training & \bf 81.8	$\pm$	0.45	&	82.0	$\pm$	2.78	&	80.4	$\pm$	1.20	&	81.4	$\pm$	0.79	\\
    & & early stop + penalty & \bf 82.4	$\pm$	0.93	&	81.4	$\pm$	2.30	&	81.3	$\pm$	0.84	&	82.0	$\pm$	0.85	\\
    \midrule
    CIVIL-Comments & accuracy
    & standard training & \bf 92.7	$\pm$	0.61	&	90.2	$\pm$	0.49	&	91.2	$\pm$	0.58	&	90.2	$\pm$	0.90	\\
    \midrule
    FMOW-Wilds & accuracy
    & standard training & \bf 52.1	$\pm$	0.27	&	51.7	$\pm$	0.39	&	51.2	$\pm$	0.39	&	51.9	$\pm$	0.46	\\
    \midrule
    PovertyMAP-Wilds & U/R Pearson R 
    & standard training & 0.725	$\pm$	0.019	&	0.711	$\pm$	0.047	&	\bf  0.814	$\pm$	0.021	&	0.755	$\pm$	0.031	\\
    \bottomrule
    \end{NiceTabular}
    \end{adjustbox}
\end{table*}

\paragraph{Dataset.}
We used Wilds~\citep{wilds2021}, which consists of various real-world data for machine learning tasks and various baseline optimization algorithms.
Each task specifies the loss function, performance metric, train-test data split, and neural net architecture.
We used Waterbirds, CIVIL-Comments, FMoW-Wilds, PovertyMAP-Wilds from Wilds.
For example, Waterbirds consists of images of two kinds of birds (landbirds and waterbirds) with different backgrounds (land and water) and the task is to predict the types of birds in images.
For further detail, see Appendix~\ref{sec:expdetails} and their original paper~\citep{wilds2021}.

\paragraph{Algorithm.}
We implemented \textsc{GDRO-EXP3P} and \textsc{GDRO-TINF} within the Python framework of Wilds.
As baseline methods, we used empirical risk minimization (ERM) and the algorithm of \citet{Sagawa2020} provided by Wilds.
We used the standard neural network architecture specified by Wilds for our learning models;
for example, ResNet50 for Waterbirds and BERT for MultiNLI and CIVIL-Comments, etc.
For $\theta$-player algorithms, we used the default optimizer with default hyperparameters in Wilds for all algorithms.
We used the official data split provided by Wilds.
We trained each model with the default number of epochs (e.g., 200 epochs for Waterbirds) in Wilds and report the performance of the best iterate.

\paragraph{Optimization heuristics.}
\citet{Sagawa2020} proposed several optimization heuristics, which were shown to improve the performance in their Waterbirds experiment.
To complement our experiments, we also report the results using these optimization heuristics in the Waterbirds and MultiNLI experiments.
In particular, we run Vanilla SGD (standard), $\ell_2$-regularization (penalty), and both early stopping and $\ell_2$-regularization (early stop+penalty).

\paragraph{Step sizes.}
For ERM and the algorithm \citet{Sagawa2020}, we used the default setting provided by Wilds.
Our algorithms used the following settings.
For $\theta$-player algorithms, we used the default optimizer with default hyperparameters in Wilds.
For $q$-player algorithms (EXP3P and TINF), we used the default step size $\eta_q = 0.01$ for the algorithm of \citet{Sagawa2020} in Wilds.

\paragraph{Mini-batch.}
We found that the following mini-batch strategy yielded the best performance.
Each mini-batch consists of $B$ samples constructed as follows:
A batch of $B$ elements is sampled from the training dataset according to the sampling strategy of the $q$-player. 
Corresponding to the indices of the sampled groups in the batch, data points are selected at random.
After constructing the mini-batch, we then update the model parameter $\theta_t$ and group weight $q_t$ using the gradient and loss averaged over the mini-batch for each group separately.
We set $B$ to the default mini-batch size provided by Wilds (e.g., $B=128$ for Waterbirds) in our experiments.

\subsubsection{Results}
We report the worst group and average test performance for each dataset in Wilds in Tables ~\ref{tab:wilds} and ~\ref{tab:wilds_avgacc} respectively.
Here, the mean and standard deviation (stddev) are computed from three independent runs with different random seeds.
In almost all datasets, \textsc{GDRO-EXP3P} and \textsc{GDRO-TINF} consistently achieved the best worst-group accuracy.
Although the performances of the algorithms except ERM are relatively close, remark that we did not tune the step size for \textsc{GDRO-EXP} and \textsc{GDRO-TINF} but used the default step size in Wilds, which is tuned for the algorithm of \citet{Sagawa2020}.


\subsection{Details of experiments in deep learning regime}\label{sec:expdetails}
We summarize the characteristics of Wilds tasks we used in our experiments in the deep learning regime.
The full details can be found in \citet{wilds2021}.

\paragraph{Waterbirds.}
The Waterbirds dataset consists of images of birds of two kinds (waterbirds and landbirds) with different backgrounds (land and water).
The task is to predict the type of birds in images.
There are $m=4$ groups corresponding to the combinations of birds and backgrounds.
The number of training examples is 4795 in total and 56 in the smallest group (waterbirds on land).
We used ResNet50 as our learning model.
We used the torch-vision implementation of ResNet50 as suggested in Wilds.

\paragraph{MultiNLI.}
The MultiNLI dataset is a natural language dataset consisting of labeled sentences.
We used the modified version of MultiNLI provided by \citet{Sagawa2020}\footnote{\url{https://github.com/kohpangwei/group_DRO}}.
Each image is assigned to $m=6$ groups corresponding to the combination of labels $\{\text{entailed, neutral, contradictory}\}$ and the existence of negation words $\{\text{no negation, negation}\}$.
The training set contains 206175 examples with 1521 examples in the smallest group (entailment with negations).
We used Hugging Face pytorch-transformers implementation of the BERT with pre-trained weights.

\paragraph{CIVIL-Comments.}
CIVIL-Comments is a natural language dataset of distribution shifts with different demographic identities.
The task is to predict whether a given text is toxic or not.
There are $m=2$ groups (toxic or not).
The learning model is BERT same as MultiNLI.

\paragraph{FMoW-wilds.}
FMoW-wilds consists of RGB satellite images of 224$\times$224 pixels.  Each image has its label (use or land) and geographical region (Africa, the Americas, Oceania, Asia, or Europe).
The task is to predict the label of a given image.
There are $m=8$ groups (the year where each image was taken).
The learning model is DenseNet121.

\paragraph{PovertyMAP-wilds.}
PovertyMAP-Wilds consists of LandSat satellite image with 8 channels (resized to 224 x 224 pixels) with a label of real-valued asset wealth index.
The task is to predict the label of a given image.
There are $m=8$ groups (the country where each image was taken).
The learning model is Resnet18ms.

%% file: main.bbl
\begin{thebibliography}{40}
\providecommand{\natexlab}[1]{#1}
\providecommand{\url}[1]{\texttt{#1}}
\expandafter\ifx\csname urlstyle\endcsname\relax
  \providecommand{\doi}[1]{doi: #1}\else
  \providecommand{\doi}{doi: \begingroup \urlstyle{rm}\Url}\fi

\bibitem[Agarwal et~al.(2012)Agarwal, Bartlett, Ravikumar, and
  Wainwright]{Agarwal2012}
Alekh Agarwal, Peter~L. Bartlett, Pradeep Ravikumar, and Martin~J. Wainwright.
\newblock Information-theoretic lower bounds on the oracle complexity of
  stochastic convex optimization.
\newblock \emph{IEEE Transactions on Information Theory}, pages 3235--3249,
  2012.

\bibitem[Akiba et~al.(2019)Akiba, Sano, Yanase, Ohta, and Koyama]{Optuna2019}
Takuya Akiba, Shotaro Sano, Toshihiko Yanase, Takeru Ohta, and Masanori Koyama.
\newblock Optuna: A next-generation hyperparameter optimization framework.
\newblock In \emph{Proceedings of the 25rd {ACM} {SIGKDD} International
  Conference on Knowledge Discovery and Data Mining}, 2019.

\bibitem[Auer et~al.(2003)Auer, Cesa-Bianchi, Freund, and Schapire]{Auer2003}
Peter Auer, Nicol{\`{o}} Cesa-Bianchi, Yoav Freund, and Robert~E. Schapire.
\newblock The nonstochastic multiarmed bandit problem.
\newblock \emph{SIAM Journal on Computing}, 32\penalty0 (1):\penalty0 48--77,
  2003.

\bibitem[Bao et~al.(2021)Bao, Chang, and Barzilay]{Bao2021}
Yujia Bao, Shiyu Chang, and Regina Barzilay.
\newblock Predict then interpolate: A simple algorithm to learn stable
  classifiers.
\newblock In \emph{Proceedings of the 38th International Conference on Machine
  Learning}, volume 139, pages 640--650, 2021.

\bibitem[Bertsimas et~al.(2018)Bertsimas, Gupta, and Kallus]{Bertsimas2018}
Dimitris Bertsimas, Vishal Gupta, and Nathan Kallus.
\newblock Data-driven robust optimization.
\newblock \emph{Mathematical Programming}, 167\penalty0 (2):\penalty0 235--292,
  2018.

\bibitem[Blanchet et~al.(2019)Blanchet, Kang, and Murthy]{Blanchet2019}
Jose Blanchet, Yang Kang, and Karthyek Murthy.
\newblock Robust wasserstein profile inference and applications to machine
  learning.
\newblock \emph{Journal of Applied Probability}, 56\penalty0 (3):\penalty0
  830--857, 2019.

\bibitem[Bubeck et~al.(2012)Bubeck, Cesa-Bianchi, et~al.]{Bubeck2012}
S{\'e}bastien Bubeck, Nicolo Cesa-Bianchi, et~al.
\newblock Regret analysis of stochastic and nonstochastic multi-armed bandit
  problems.
\newblock \emph{Foundations and Trends{\textregistered} in Machine Learning},
  5\penalty0 (1):\penalty0 1--122, 2012.

\bibitem[Buolamwini and Gebru(2018)]{Buolamwini18a}
Joy Buolamwini and Timnit Gebru.
\newblock Gender shades: Intersectional accuracy disparities in commercial
  gender classification.
\newblock In \emph{Proceedings of the 1st Conference on Fairness,
  Accountability and Transparency}, pages 77--91, 2018.

\bibitem[Carmon et~al.(2021)Carmon, Jambulapati, Jin, and Sidford]{Carmon21a}
Yair Carmon, Arun Jambulapati, Yujia Jin, and Aaron Sidford.
\newblock Thinking inside the ball: Near-optimal minimization of the maximal
  loss.
\newblock In \emph{Proceedings of 34th Conference on Learning Theory}, volume
  134 of \emph{Proceedings of Machine Learning Research}, pages 866--882, 2021.

\bibitem[Cesa-Bianchi and Lugosi(2006)]{Cesa-Bianchi2006}
Nicolo Cesa-Bianchi and Gabor Lugosi.
\newblock \emph{Prediction, Learning, and Games}.
\newblock Cambridge University Press, 2006.

\bibitem[Curi et~al.(2020)Curi, Levy, Jegelka, and Krause]{Curi2020}
Sebastian Curi, Kfir~Y. Levy, Stefanie Jegelka, and Andreas Krause.
\newblock Adaptive sampling for stochastic risk-averse learning.
\newblock In \emph{Advances in Neural Information Processing Systems}, pages
  1036--1047, 2020.

\bibitem[Diana et~al.(2021)Diana, Gill, Globus-Harris, Kearns, Roth, and
  Sharifi-Malvajerdi]{Diana2021}
Emily Diana, Wesley Gill, Ira Globus-Harris, Michael Kearns, Aaron Roth, and
  Saeed Sharifi-Malvajerdi.
\newblock Lexicographically fair learning: Algorithms and generalization.
\newblock In \emph{Proceedings of the 2nd Symposium on Foundations of
  Responsible Computing}, pages 6:1--6:23, 2021.
\newblock \doi{10.4230/LIPIcs.FORC.2021.6}.

\bibitem[Dua and Graff(2017)]{Dua:2019}
Dheeru Dua and Casey Graff.
\newblock {UCI} machine learning repository, 2017.
\newblock URL \url{http://archive.ics.uci.edu/ml}.

\bibitem[Duchi and Namkoong(2021)]{Duchi2021}
John~C. Duchi and Hongseok Namkoong.
\newblock Learning models with uniform performance via distributionally robust
  optimization.
\newblock \emph{The Annals of Statistics}, 49\penalty0 (3):\penalty0 1378 --
  1406, 2021.

\bibitem[Esfahani and Kuhn(2018)]{Esfahani2018}
Peyman~Mohajerin Esfahani and Daniel Kuhn.
\newblock Data-driven distributionally robust optimization using the
  wasserstein metric: Performance guarantees and tractable reformulations.
\newblock \emph{Mathematical Programming}, 171\penalty0 (1):\penalty0 115--166,
  2018.

\bibitem[Goh and Sim(2010)]{Goh2010}
Joel Goh and Melvyn Sim.
\newblock Distributionally robust optimization and its tractable
  approximations.
\newblock \emph{Operations Research}, 58\penalty0 (4-part-1):\penalty0
  902--917, 2010.

\bibitem[Haghtalab et~al.(2022)Haghtalab, Jordan, and Zhao]{Haghtalab2022}
Nika Haghtalab, Michael Jordan, and Eric Zhao.
\newblock On-demand sampling: Learning optimally from multiple distributions.
\newblock \emph{Advances in Neural Information Processing Systems},
  35:\penalty0 406--419, 2022.

\bibitem[Hashimoto et~al.(2018)Hashimoto, Srivastava, Namkoong, and
  Liang]{Hashimoto18a}
Tatsunori Hashimoto, Megha Srivastava, Hongseok Namkoong, and Percy Liang.
\newblock Fairness without demographics in repeated loss minimization.
\newblock In \emph{Proceedings of the 35th International Conference on Machine
  Learning}, pages 1929--1938, 2018.

\bibitem[Hazan(2016)]{Hazan2016book}
Elad Hazan.
\newblock \emph{Introduction to Online Convex Optimization}.
\newblock 2016.

\bibitem[Hovy and S{\o}gaard(2015)]{Hovy2015}
Dirk Hovy and Anders S{\o}gaard.
\newblock Tagging performance correlates with author age.
\newblock In \emph{Proceedings of the 53rd Annual Meeting of the Association
  for Computational Linguistics and the 7th International Joint Conference on
  Natural Language Processing}, pages 483--488, 2015.

\bibitem[Hu et~al.(2018)Hu, Niu, Sato, and Sugiyama]{Hu2018}
Weihua Hu, Gang Niu, Issei Sato, and Masashi Sugiyama.
\newblock Does distributionally robust supervised learning give robust
  classifiers?
\newblock In \emph{Proceedings of the 35th International Conference on Machine
  Learning}, pages 2029--2037, 2018.

\bibitem[Jin et~al.(2021)Jin, Zhang, Wang, and Wang]{Jikai2021}
Jikai Jin, Bohang Zhang, Haiyang Wang, and Liwei Wang.
\newblock Non-convex distributionally robust optimization: Non-asymptotic
  analysis.
\newblock In \emph{Advances in Neural Information Processing Systems},
  volume~34, pages 2771--2782, 2021.

\bibitem[Jurgens et~al.(2017)Jurgens, Tsvetkov, and Jurafsky]{Jurgens2017}
David Jurgens, Yulia Tsvetkov, and Dan Jurafsky.
\newblock Incorporating dialectal variability for socially equitable language
  identification.
\newblock In \emph{Proceedings of the 55th Annual Meeting of the Association
  for Computational Linguistics}, pages 51--57, 2017.

\bibitem[Kirschner et~al.(2020)Kirschner, Bogunovic, Jegelka, and
  Krause]{Kirschner20a}
Johannes Kirschner, Ilija Bogunovic, Stefanie Jegelka, and Andreas Krause.
\newblock Distributionally robust bayesian optimization.
\newblock In \emph{Proceedings of the 33rd International Conference on
  Artificial Intelligence and Statistics}, pages 2174--2184, 2020.

\bibitem[Koh et~al.(2021)Koh, Sagawa, Marklund, Xie, Zhang, Balsubramani, Hu,
  Yasunaga, Phillips, Gao, Lee, David, Stavness, Guo, Earnshaw, Haque, Beery,
  Leskovec, Kundaje, Pierson, Levine, Finn, and Liang]{wilds2021}
Pang~Wei Koh, Shiori Sagawa, Henrik Marklund, Sang~Michael Xie, Marvin Zhang,
  Akshay Balsubramani, Weihua Hu, Michihiro Yasunaga, Richard~Lanas Phillips,
  Irena Gao, Tony Lee, Etienne David, Ian Stavness, Wei Guo, Berton~A.
  Earnshaw, Imran~S. Haque, Sara Beery, Jure Leskovec, Anshul Kundaje, Emma
  Pierson, Sergey Levine, Chelsea Finn, and Percy Liang.
\newblock {WILDS}: A benchmark of in-the-wild distribution shifts.
\newblock In \emph{International Conference on Machine Learning (ICML)}, 2021.

\bibitem[Lim and Wright(2016)]{Lim2016}
Cong~Han Lim and Stephen~J. Wright.
\newblock Efficient bregman projections onto the permutahedron and related
  polytopes.
\newblock In \emph{Proceedings of the 19th International Conference on
  Artificial Intelligence and Statistics}, pages 1205--1213, 2016.

\bibitem[Martinez et~al.(2021)Martinez, Bertran, Papadaki, Rodrigues, and
  Sapiro]{Martinez2021}
Natalia~L Martinez, Martin~A Bertran, Afroditi Papadaki, Miguel Rodrigues, and
  Guillermo Sapiro.
\newblock Blind pareto fairness and subgroup robustness.
\newblock In \emph{Proceedings of the 38th International Conference on Machine
  Learning}, pages 7492--7501, 2021.

\bibitem[Namkoong and Duchi(2016)]{Namkoong2016}
Hongseok Namkoong and John~C Duchi.
\newblock Stochastic gradient methods for distributionally robust optimization
  with {$f$}-divergences.
\newblock In \emph{Advances in Neural Information Processing Systems}, 2016.

\bibitem[Orabona(2019)]{Orabona2019book}
Francesco Orabona.
\newblock A modern introduction to online learning.
\newblock \emph{arXiv preprint arXiv:1912.13213}, 2019.

\bibitem[Oren et~al.(2019)Oren, Sagawa, Hashimoto, and Liang]{Oren2019}
Yonatan Oren, Shiori Sagawa, Tatsunori Hashimoto, and Percy Liang.
\newblock Distributionally robust language modeling.
\newblock In \emph{Proceedings of the Conference on Empirical Methods in
  Natural Language Processing and the 9th International Joint Conference on
  Natural Language Processing (EMNLP-IJCNLP)}, pages 4227--4237, 2019.

\bibitem[Qi et~al.(2021)Qi, Guo, Xu, Jin, and Yang]{Qi2021}
Qi~Qi, Zhishuai Guo, Yi~Xu, Rong Jin, and Tianbao Yang.
\newblock An online method for a class of distributionally robust optimization
  with non-convex objectives.
\newblock In \emph{Advances in Neural Information Processing Systems},
  volume~34, pages 10067--10080, 2021.

\bibitem[Rakhlin and Sridharan(2013)]{Rakhlin2013a}
Alexander Rakhlin and Karthik Sridharan.
\newblock Optimization, learning, and games with predictable sequences.
\newblock In \emph{Advances in Neural Information Processing Systems}, 2013.

\bibitem[Sagawa et~al.(2020)Sagawa, Koh, Hashimoto, and Liang]{Sagawa2020}
Shiori Sagawa, Pang~Wei Koh, Tatsunori~B. Hashimoto, and Percy Liang.
\newblock Distributionally robust neural networks for group shifts: On the
  importance of regularization for worst-case generalization.
\newblock In \emph{The 8th International Conference on Learning
  Representations}, 2020.

\bibitem[Scarf(1958)]{Scarf1958}
Herbert Scarf.
\newblock A min-max solution of an inventory problem.
\newblock \emph{Studies in the mathematical theory of inventory and
  production}, 1958.

\bibitem[Staib and Jegelka(2019)]{Staib19b}
Matthew Staib and Stefanie Jegelka.
\newblock Distributionally robust optimization and generalization in kernel
  methods.
\newblock In \emph{Advances in Neural Information Processing Systems}, 2019.

\bibitem[Staib et~al.(2019)Staib, Wilder, and Jegelka]{Staib19a}
Matthew Staib, Bryan Wilder, and Stefanie Jegelka.
\newblock Distributionally robust submodular maximization.
\newblock In \emph{Proceedings of the 22nd International Conference on
  Artificial Intelligence and Statistics}, pages 506--516, 2019.

\bibitem[Williamson and Menon(2019)]{Williamson2019a}
Robert Williamson and Aditya Menon.
\newblock Fairness risk measures.
\newblock In \emph{Proceedings of the 36th International Conference on Machine
  Learning}, pages 6786--6797, 2019.

\bibitem[Zhang et~al.(2021)Zhang, Menon, Veit, Bhojanapalli, Kumar, and
  Sra]{Zhang2021}
Jingzhao Zhang, Aditya~Krishna Menon, Andreas Veit, Srinadh Bhojanapalli,
  Sanjiv Kumar, and Suvrit Sra.
\newblock Coping with label shift via distributionally robust optimisation.
\newblock In \emph{The 9th International Conference on Learning
  Representations}, 2021.

\bibitem[Zimmert and Seldin(2021)]{Zimmert2021}
Julian Zimmert and Yevgeny Seldin.
\newblock Tsallis-inf: An optimal algorithm for stochastic and adversarial
  bandits.
\newblock \emph{Journal of Machine Learning Research}, 22\penalty0
  (28):\penalty0 1--49, 2021.

\bibitem[Zinkevich(2003)]{Zinkevich2003}
Martin Zinkevich.
\newblock Online convex programming and generalized infinitesimal gradient
  ascent.
\newblock In \emph{Proceedings of the 20th International Conference on
  International Conference on Machine Learning}, pages 928--935, 2003.

\end{thebibliography}
